\documentclass{article}

\usepackage[margin=1in]{geometry} 
\usepackage{lipsum}
\usepackage{enumitem}

\usepackage[utf8]{inputenc} 
\usepackage[T1]{fontenc}    
\usepackage{url}            
\usepackage{booktabs}       
\usepackage{amsfonts}       
\usepackage{nicefrac}       
\usepackage{microtype}      
\usepackage{graphicx}
\usepackage{natbib}
\usepackage{doi}
\usepackage{float}
\usepackage{amsmath}
\usepackage{amsthm}
\usepackage{multirow}
\usepackage{algorithm}
\usepackage{algpseudocode}
\usepackage{xcolor}
\usepackage{hyperref}
\hypersetup{
    colorlinks,
    linkcolor={red!50!black},
    citecolor={blue!50!black},
    urlcolor={blue!80!black}
}

\setcitestyle{authoryear, open={(},close={)}}

\newcommand{\beginsupplement}{%
        \setcounter{table}{0}
        \renewcommand{\thetable}{S\arabic{table}}%
        \renewcommand{\theHtable}{S\arabic{figure}}
        \setcounter{figure}{0}
        \renewcommand{\thefigure}{S\arabic{figure}}%
        \renewcommand{\theHfigure}{S\arabic{figure}}
        \setcounter{section}{0}
        \renewcommand{\thesection}{S\arabic{section}}%
        \renewcommand{\theHsection}{S\arabic{section}}%
     }

\title{\textbf{Know Your Limits: Entropy Estimation Modeling for Compression and Generalization}}

\author{\small \href{https://orcid.org/0000-0003-1661-4579}{\color{black} \textbf{Benjamin L. Badger} \thanks{The authors would like to thank IBM for support during the writing of this paper. Code may be found on \url{https://github.com/blbadger/entropyestimators}}, \href{https://orcid.org/0000-0002-7028-4569}{\color{black} \hspace{1mm} \textbf{Matthew Neligeorge}}} \\
	\small{IBM} \\
	\small{\texttt{ben.badger@ibm.com}} \\
}

\date{}

\hypersetup{
pdftitle={Entropy Estimation Models},
pdfsubject={Deep Learning},
pdfauthor={Benjamin L. Badger, Matthew Neligeorge},
pdfkeywords={Compression, Entropy, Transformers, Mixers, Representation},
}

\begin{document}
\maketitle

\begin{abstract}\normalsize{
    Language prediction is constrained by informational entropy intrinsic to language, such that there exists a limit to how accurate any language model can become and equivalently a lower bound to language compression. The most efficient language compression algorithms today are causal (next token prediction) large language models, but the use of these models to form accurate estimates of language entropy is currently computationally infeasible. We introduce encoder-augmented causal decoder model architectures that exhibit superior training efficiency characteristics and achieve higher compression than causal transformers even when trained on modest hardware. We demonstrate how entropy estimates can be obtained on a per-token basis, and show that the generalization of models trained to approach the entropy of their training data necessarily exceeds the generalization of models trained to minimize loss beyond this value. We show empirically that causal models trained to approach but not exceed estimated per-token entropies exhibit greater generalization than models trained without taking entropy into account.}
\end{abstract}

\section{Introduction}

    How good can a language model possibly be, and how can we use this knowledge to assist with training models? We approach this question from the perspective of entropy intrinsic to language, which can be thought of as the amount of uncertainty inherent in the prediction of any next element in a language encoding (characters, words, tokens, or otherwise). It is clear that language has non-zero entropy, as there for any given segment of text one has a certain amount of free will in completing most portions of that segment.

    Language prediction, entropy, and compression have been intimately connected since the early work by Shannon \citep{shannon1948mathematical} and Wiener \citep{wiener2019cybernetics}. In particular, Shannon's definition of informational entropy allows for the interconversion of prediction fidelity and compression value, originally defined in bits per character \citep{shannon1951prediction}. There are of course non-predictive compression algorithms; undercomplete autoencoders that contain a smaller embedding than input, rule-based approaches that use n-gram statistics to compress redundant characters, and generalized sequence compression approaches not tailored to language. The long-standing hypothesis that compression is analogous to intelligence (or at least generalizable model abilities) has been frequently supported over the last half century. 

    By far the most effective language compression algorithms today are large causal language models. For lossless compression of Wikipedia text, methods based on causal transformers trained on somewhat limited hardware (0.85 bits required per input byte, hereafter referred to as `BPB') have been shown to far outperform compression methods such as \texttt{gzip} (2.58 BPB) \citep{bellard2019lossless, mahoneylargetext} on that dataset. Even given extremely limited (CPU-only) inference constraints, deep learning approaches such as LSTMs have been shown to achieve the highest compression of all tested methods on similar datasets \citep{hutterprize, knollcmix}. As most machine learning model capabilities improve with increases in data and compute, it is unsurprising that training larger models on more tokens results in greater compression. For example, on a more difficult compression corpus of excerpts from the Pile \citep{gao2020pile}, the 671 billion parameter mixture-of-experts Deepseek V3 trained on 15 trillion tokens yields 0.55 BPB, and the dense 405 billion parameter Llama 3.1 achieves 0.54 BPB \citep{deepseekai2025deepseekv3technicalreport}. 

    This bears the question: is there an approach to language compression that is more efficient to train than a large causal transformer, but that can scale to make use of the same amount of compute as these models in order to reach a lower compression ratio? As compression is closely related to intrinsic entropy, a more efficient compression model would also be a more accurate entropy estimator. In this work we introduce an architecture that we show to be more efficient than causal transformers for compression and entropy estimation. We then show how this model may be used to calculate entropies of individual tokens, prove that training a model to exceed its training dataset's entropy results in worse generalization, and show how token-specific entropy may be used for superior generalization when training causal models.

\section{Our Contribution}

    The fundamental idea of the first portion of this paper may be summarized as follows:
    \vspace{0.15cm}

    \noindent{\textit{By combining an encoder with a causal decoder we can capture much of the intrinsic entropy of any sequence in that encoder, resulting in more efficient training for entropy estimation than exists using causal decoder-only architectures. 
    }}
    \vspace{0.15cm}

    \noindent{and the second portion as}
    
    \vspace{0.15cm}
    \noindent{\textit{Entropy estimates may be obtained on a per-token basis, and causal language models generalize better when trained using this information.
    }}
    \vspace{0.15cm}
    
    \noindent{In this work we introduce the following:}
    \vspace{0.15cm}

    \begin{enumerate}[nosep]
    \item Autoencoder architectures for language
    \item Embedding-augmented causal architectures for more efficient language compression
    \item A simple noise injection method for simulation of quantization-aware training
    \item Per-token entropy computation methods for non-causal models, and architectures to accurately guess these estimates
    \item A proof that training models using entropy estimates results in ideal generalization in a certain sense, and empirical evidence that entropy-informed training results in greater generalization.
    \end{enumerate}
    \vspace{0.15cm}
    We note that while our investigation is focused on language in the spirit of Shannon information theory, these techniques may be easily applied to vision, audio, or any other domain in which one wishes to model sequences of tokens.

\subsection{Related Work}
    
    We take much inspiration from the pioneering work of Shannon \citep{shannon1948mathematical}, and use the mathematical framework established there and by others \citep{wiener2019cybernetics} throughout this paper. An early attempt to estimate the per-character entropy of a segment of English text is found in \citep{shannon1951prediction}, where the author used a reformulation of cross-entropy applied to a person guessing letters in a sequence from a previously-unread text segment in order to estimate the entropy of that text segment. We employ a similar framework for calculating per-token entropy estimates, modified for use with causal and noncausal models.

    A contribution towards the re-introduction of compression as a language modeling metric for large datasets is found in \citep{gao2020pile800gbdatasetdiverse}. We employ a number of conventions popularized by that work, including the use of bits per input byte as a measure of compression and entropy estimation. \citep{delétang2024languagemodelingcompression} provides another view into language modeling in the framework of compression, and finds that these models also extend their compression abilities to vision and speech datasets, albeit to a more limited degree. 

    Wei and colleagues observed that near-lossless text compression can be achieved at a 1:10 ratio using transformer based encoder-decoder vision language models \citep{wei2025deepseekocrcontextsopticalcompression}.  The efficiency of training what we call a `global' encoder (where all input tokens are attended, rather than only previous tokens) in the vision-language setting is presented in the context of its benefits for compression, but the relationship between the use of a global encoder with training efficiency due to the entropy of the source is not explored in that work. We provide a perspective here that it is likely the global encoder paired with a decoder of that vision model, not the input modality per se, that is behind the impressive compression capabilities of that model.
    
    Consistent methods for introducing entropy for classification datasets are found in \citep{shalev2020neuraljointentropyestimation}. A notable departure here is that our methods are designed with essentially the opposite goal with respect to training properties; specifically we seek more efficient convergence with respect to compute applied when given an extremely large dataset, rather than optimal generalization when given a small dataset. Equivalently, we seek a model that fundamentally scales in superior ways to the causal models currently used, whose scaling is detailed in \citep{hoffmann2022trainingcomputeoptimallargelanguage}. We focus primarily on scaling in terms of tokens trained, but also investigate scaling with model size.

\section{Autoencoders for Language Compression} \label{autosection}

    A direct method of compression and equivalently entropy estimation is an autoencoder where the encoder's embedding is smaller than the input. For autoencoder $\theta$ trained to convergence, the model's estimate of the entropy of text corpus $x$ may be estimated as shown in Equation \ref{eq1}, where $|e|$ signifies the size of the embedding in bits as computed by multiplying the number of activations by the number of bits per embedding activation ($n_p*b_p$), $L_t$ is the length of the corpus in tokens, and $\Bbb L(O(x, \theta)_i, x_i)$ the cross-entropy between the autoencoder model output and sample $x_i$.

    \begin{equation}\
    H(x) = \frac{|e|}{L_t} * \ln(2) + H(O(x, \theta), x) = \frac{|e|}{L_t} * \ln(2) + \sum_i \Bbb L(O(x, \theta)_i, x_i) 
    \label{eq1}
    \end{equation}
    
    We first explore language model autoencoders using two architectures capable of causal language modeling: the transformer \citep{vaswani2023attentionneed} and the masked mixer \citep{badger2025maskedmixerslanguagegeneration}, which substitutes self-attention for masked convolutions (see Figure \ref{fig1} for the autoencoder architecture we employ for our experiments). We explore architectural variations that minimize the loss term in Equation \ref{eq1} given a limited amount of compute, and train on excerpts of the FineWeb-Edu dataset, a subset of the FineWeb dataset that was itself generated via filtering the Common Crawl and is representative of the large, diverse, and moderately curated datasets typical of frontier model training today \citep{penedo2024finewebdatasetsdecantingweb}. We also use a mathematical subset of the FineWeb, FineMath 4+ \citep{allal2025smollm2smolgoesbig}, as an example of a lower-entropy dataset as is typical of a mathematical and programming corpus. We optimize Cross-Entropy Loss objective functions using AdamW \citep{loshchilov2019decoupledweightdecayregularization} unless otherwise noted, and typically train for 200,000 steps for each training run. Throughout this work, we train models of between 75 and 250 million parameters on datasets of between 13 and 30 billion tokens. We match models of different architectures by approximate compute requirements, such that for any figure panel in this work all models exhibit approximately the same throughput and require approximately the same device memory for a fixed batch size.
    
    \subsection{Transformers Struggle to Autoencode using Repeat Embeddings}
    
    Prior work showed that masked mixer-based autoencoders are more efficient to train than transformers when the decoder receives repeated embeddings from the encoder, one per token position of the input \citep{badger2025maskedmixerslanguagegeneration}. We further characterize this difference in training efficiency for scaled models, and find it to be consistent regardless of model size (Figure \ref{figs1}). We hypothesized that the poor transformer training in this paradigm results from an inability of positional encodings to sufficiently differentiate repeated token embeddings, which would result in insufficient use of the embedding's information by the decoder (see Section \ref{sections1} for a more formal explanation of this argument).

    \begin{equation}
        E(x) = W(x_{m :m + s} \circ x_{0: \;\max (0, \; m + s - d_m)}) + \beta
        \label{eq6}
    \end{equation}
    
    Inspired by \citep{morris2023languagemodelinversion}, we tested this hypothesis by unrolling the encoder's embedding such that each decoder input is unique modulo some constant. For inputs of arbitrary token dimension given fixed $d_m$ size, given the reassigned index $m = n \mod d_m$ and embedding source sequence size $s$ we use Equation \ref{eq6}. We use $s = d_m /2$ for our experiments, and throughout this paper we signify the tensor $x$ indexed by $(x_i, x_{i+1}, x_{i+2}, ..., x_{j})$ as $x_{i: j}$ and denote concatenation of vectors $a, b$ as $a \circ b$. In support of our hypothesis, we find that replacing the decoder's input with unrolled projections of the encoder's embedding results in vastly improved transformer autoencoder training characteristics as shown in Figure \ref{fig1}.

    \begin{figure}[h]
        \centering
        \includegraphics[width=0.99\textwidth]{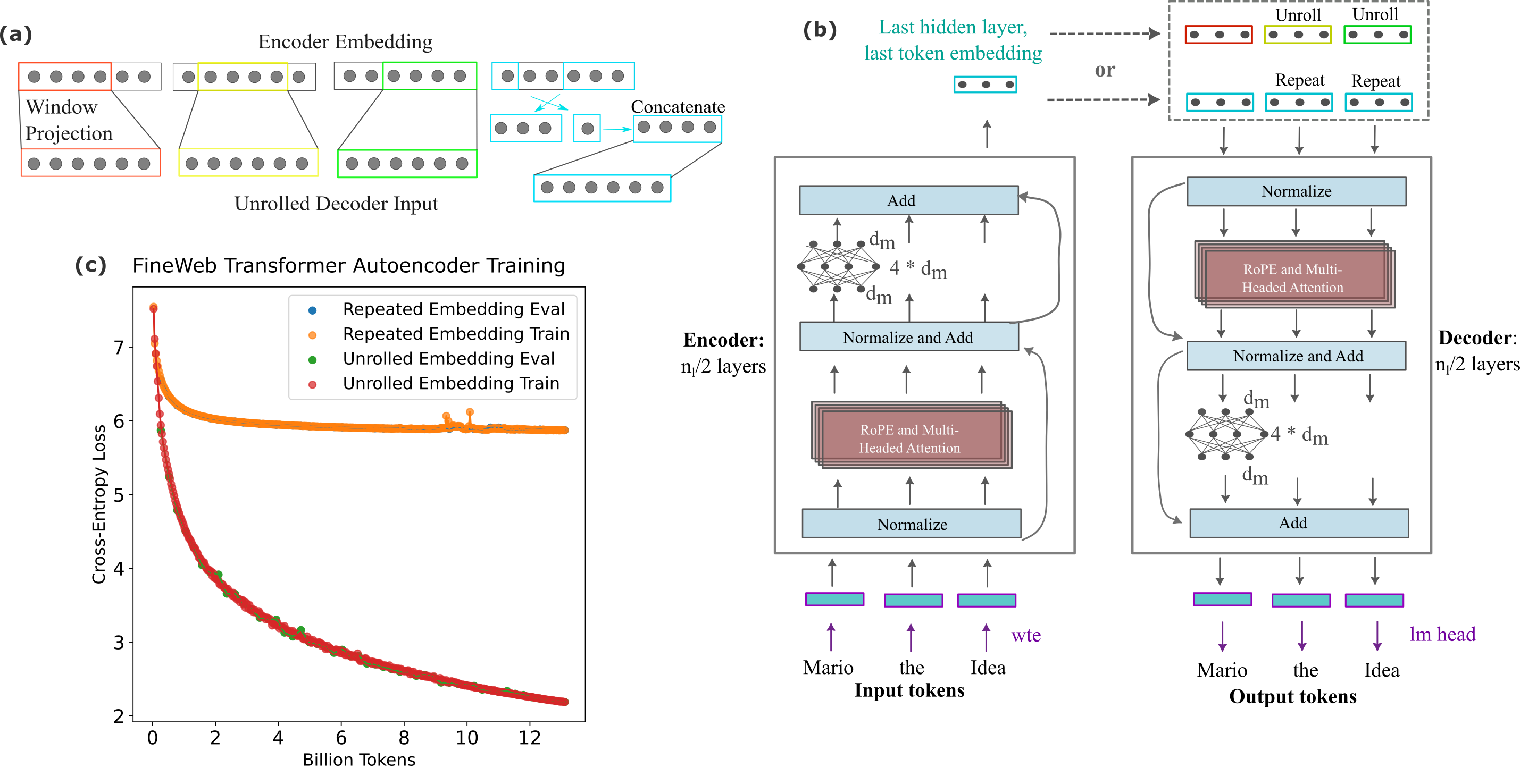}
        \caption{Transformer Autoencoders are poorly trainable with repeated but not unrolled embeddings. (a) Embedding unrolling method. (b) Experimental design and transformer autoencoder architecture. (c) Autoencoder training efficiencies on FineWeb-edu (number of layers, model dimension (width), and token context window $n_l=16, d_m=512, n_{ctx}=512$, respectively)}
        \label{fig1}
    \end{figure}

    \subsection{Autoencoder Architecture Optimizations}

    Optimizing transformer and mixer autoencoder architectures for training efficiency, we find that causal masking is necessary for stable training of mixer autoencoders (Figure \ref{figs3}), and that for mixers increasing the kernel size but not head number results in increases in autoencoder efficiency without sacrificing numerical stability using FP16/FP32 mixed precision training (Figure \ref{figs2}, otherwise BP16/FP32 precision is required).  We observe similar training efficiencies between identically sized mixer and transformer autoencoders after accounting for throughput differences (mixers contain far fewer activations per forward and backward pass and exhibit around 2x the throughput for a model of a given layer number and dimension). Interestingly there are much larger changes in training efficiency for autoencoders than for causal mixers with different head numbers (Figure \ref{figs4}). We observe that simply increasing the number of inter-token parameters or activations does not necessarily lead to increases in training efficiency and in some cases actually reduces this metric.
    
    We further explored mixers with convolutions with kernels $k>1$, which mixes both sequence and a limited number of hidden dimension components in a single transformation. We observe non-obvious relationships between kernel size and training efficiency: for repeat embedding introduction the optimal kernel size is identical to the optimal head size (Figure \ref{figs2} (b)) but for unrolled embeddings we observe an increase in per-step loss with an increase in kernel size (up to $k=16$ as shown in Figure \ref{figs2} (d), where throughput begins to slow considerably). 

    \subsection{Autoencoders are Inefficient to Train as Compression Models but exhibit Superior Scaling Characteristics compared to Causal Models}

    \begin{figure}[h]
        \centering
        \includegraphics[width=0.85\textwidth]{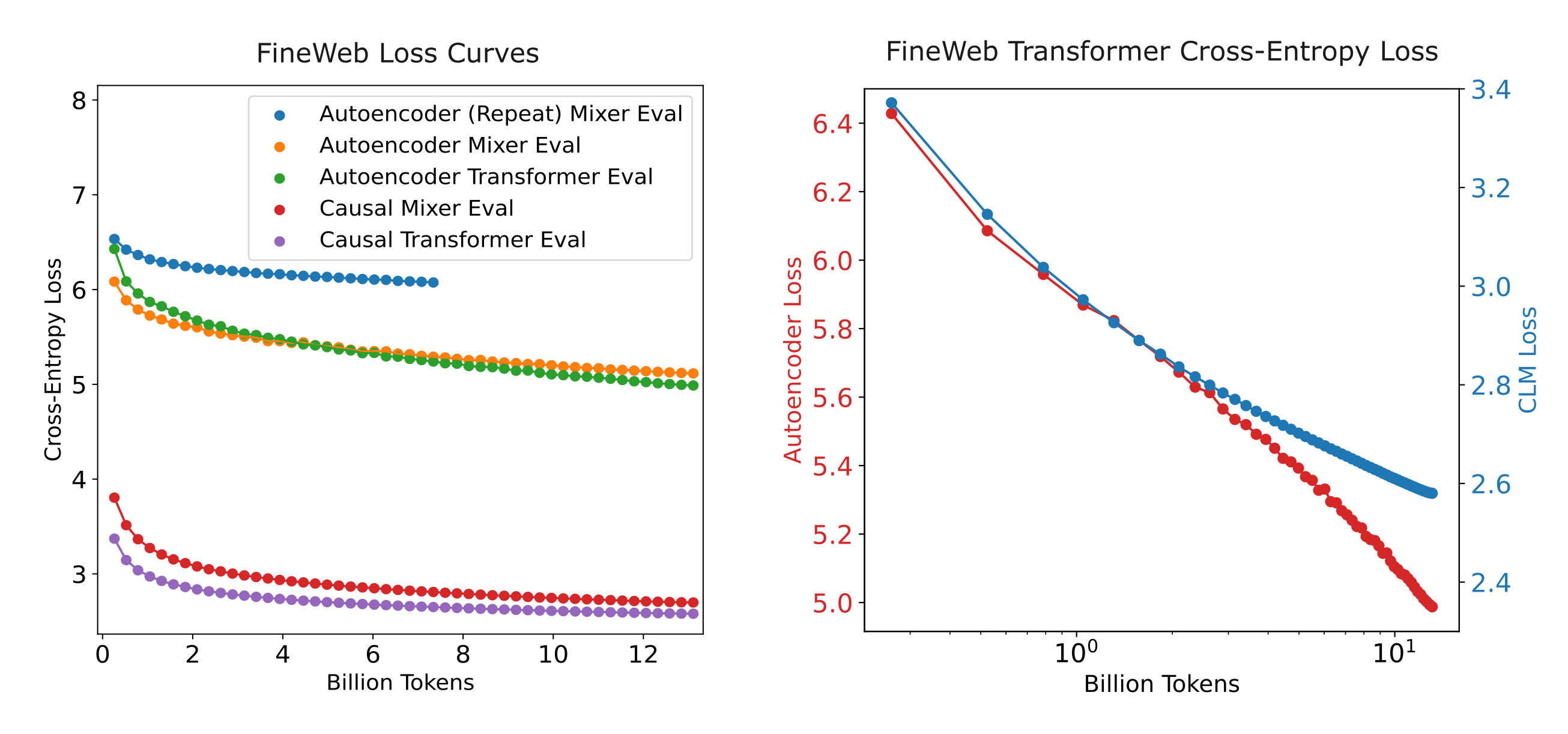}
        \caption{Compressive autoencoder and causal language model training characteristics on FineWeb. All autoencoders are $d_m=512, n_l=16, n_{ctx}=512$ with a compressed embedding of size $d_e=128$, causal models are compute-matched to the autoencoders.}
        \label{fig2}
    \end{figure}

    In the best case, the autoencoders introduced in the last section approach a compression of $n_{ctx}:1 = 512:1$ with respect to the size of all embeddings of input tokens compared to the embedding passed to the decoder. In those models this compression is effectively negated by the token embedding transformation, where a token representable in 16 bits is converted to an embedding requiring thousands of bits. We therefore explored the training efficiency of autoencoders where the encoder's embedding is subject to an undercomplete (compressive) linear transformation in order to increase the total model compression.

    We measure compression in Bits per Input Byte, BPB, after \citep{gao2020pile} by adapting their computation to cases where a fixed-size embedding is included for a fixed-size context window. For convenience, we restate the computation of a causal model's bits per byte in terms of the length of the text to compress in tokens $L_t$, the number of bytes the uncompressed text requires $L_b$, and the cross-entropy loss of a model in question with that text $\Bbb L$ in Equation \ref{eq30}.

    \begin{equation}
        \mathtt{BPB} = (L_t / L_b) \Bbb L  /\ln(2)
        \label{eq30}
    \end{equation}

    We make the simplifying assumption that our UTF-8 encoded text datasets require one byte per character (in reality 1 and 4 bytes are required) and calculate the amortized bits per byte required for an embedding of size $n_p$ with $b_p$ bits per activation used to represent that embedding, a context window of size $n_{ctx}$ and $L_b / L_t$ bytes per input token as shown in Equation \ref{eq31}. Solving Equation \ref{eq30} in terms of cross entropy loss, we can amortize the embedding's information in terms of cross-entropy loss over all predicted tokens as shown in Equation \ref{eq32}.

    \begin{equation}
         \mathtt{BPB_a} = \frac{n_p * b_p}{n_{ctx} * (L_b / L_t)}
         \label{eq31}
    \end{equation}

    \begin{equation}
        \begin{split}
         \Bbb L_a &= (\mathtt{BPB_a} * \ln(2))/(L_t / L_b) \\
         \Bbb L &= \Bbb L(O(x, \theta), x) + \Bbb L_a
        \end{split}
    \label{eq32}
    \end{equation}

    For the tokenizer used throughout this paper (with size 8k), we find a $L_b / L_t$ of 3.92 for FineWeb and 2.82 for FineMath. When we compute the total loss of our lowest-prediction-loss autoencoder in the last section ($n_{ctx}=512, \; d_e = 1024, \; \Bbb L=0.3$ assuming 4 bits per parameter, we find a BPB of 2.04 which corresponds to normalized loss of $\Bbb L = \Bbb L + \Bbb L_a =  0.3 + 5.54$ which is far greater than the corresponding causal language model loss give the same compute ($\Bbb L = 2.6$). Thus we investigate the training efficiency of autoencoders with greater compression in the embedding.
    
    We find that both transformer and mixer autoencoders are relatively inefficient to train even with a relatively mild compression (encoder to embedding to decoder hidden sizes of $d_e=512 \to d_c = 128 \to d_d=512$ via linear transformations) and exhibit high cross-entropy loss relative to uncompressed autoencoders (Figure \ref{fig2} left) even without taking into account the amortized loss (0.51 BPB, or $\Bbb L_a = 1.386$ assuming 8 bits per embedding activation). At the same time, we observe superior asymptotic training efficiency characteristics in terms of loss change per billion tokens seen for the autoencoder (Figure \ref{fig2} right) such that with more compute and data it is likely that these autoencoder architectures would yield significant compression values.
    
\section{Entropy Estimation Model Architecture and Optimization}

    Causal language models can be used for lossless compression as follows: one takes a trained model, inferences it on a text corpus, and determines the bits to correct any errors in the model's predictions for each next element as they appear in a text segment, and saves those bits as the compressed segment. The compression value is the total number of bits per sequence byte, or in terms of entropy as given in Equation \ref{eq2}, where $x'$ denotes a right shifted $x$.

    \begin{equation}\
    H(x) = H(O(x, \theta), x') = \sum_i \Bbb L(O(x_{:i}, \theta)_i, x_{i+1}) \\
    \label{eq2}
    \end{equation}

    Comparing the total compression achieved by the autoencoders introduced in the last section to causal models, we find that the latter are vastly more efficient with respect to entropy reached on a certain corpus. We reasoned that this was largely the consequence of two differences in how the decoder processes input information: first, the decoder is responsible for generating only one token per forward pass in the causal model case instead of all tokens in one forward pass as occurs in the autoencoder, and secondly the decoder receives all previous token information exactly in the causal case but only a (compressed) representation of the entire input in the undercomplete autoencoder.

    We conjectured that combining an encoder with a causal decoder would ameliorate both of these challenges. This encoder would take information from the entire input and in that way be globally attentive, such that the entire encoder-decoder model is no longer causal and cannot be used for auto-regressive token generation. Entropy values per sequence are identical to those for autoencoders without causal decoders (Equation \ref{eq1}). We call this architecture the Entropy Estimation Model after its primary application. We do not explore the use of this architecture for non-autoregressive generation, but note that it would be amenable to masked language modeling and denoising diffusion-based generation, among other approaches.

    \subsection{Embedding Introduction and Training Efficiency}

    In one sense the goal behind the entropy estimation model architecture is to be able to substantially shrink the encoder's embedding size relative to what is necessary for an autoencoder, while exceeding the training efficiencies of causal models due to the next token information present in the embedding. Because of the importance of this embedding, we investigated three approaches to the introduction of the embedding to the causal decoder: token concatenation, embedding concatenation, and embedding projection. We find that mixers are somewhat more efficient to train when embedding concatenation is used but that transformers exhibit minimal training efficiency differences between these introduction choices (Figure \ref{fig3}). As such, we use embedding concatenation for masked mixers and token concatenation for transformers in other figures of this work.

    \begin{figure}[h]
        \centering
        \includegraphics[width=0.86\textwidth]{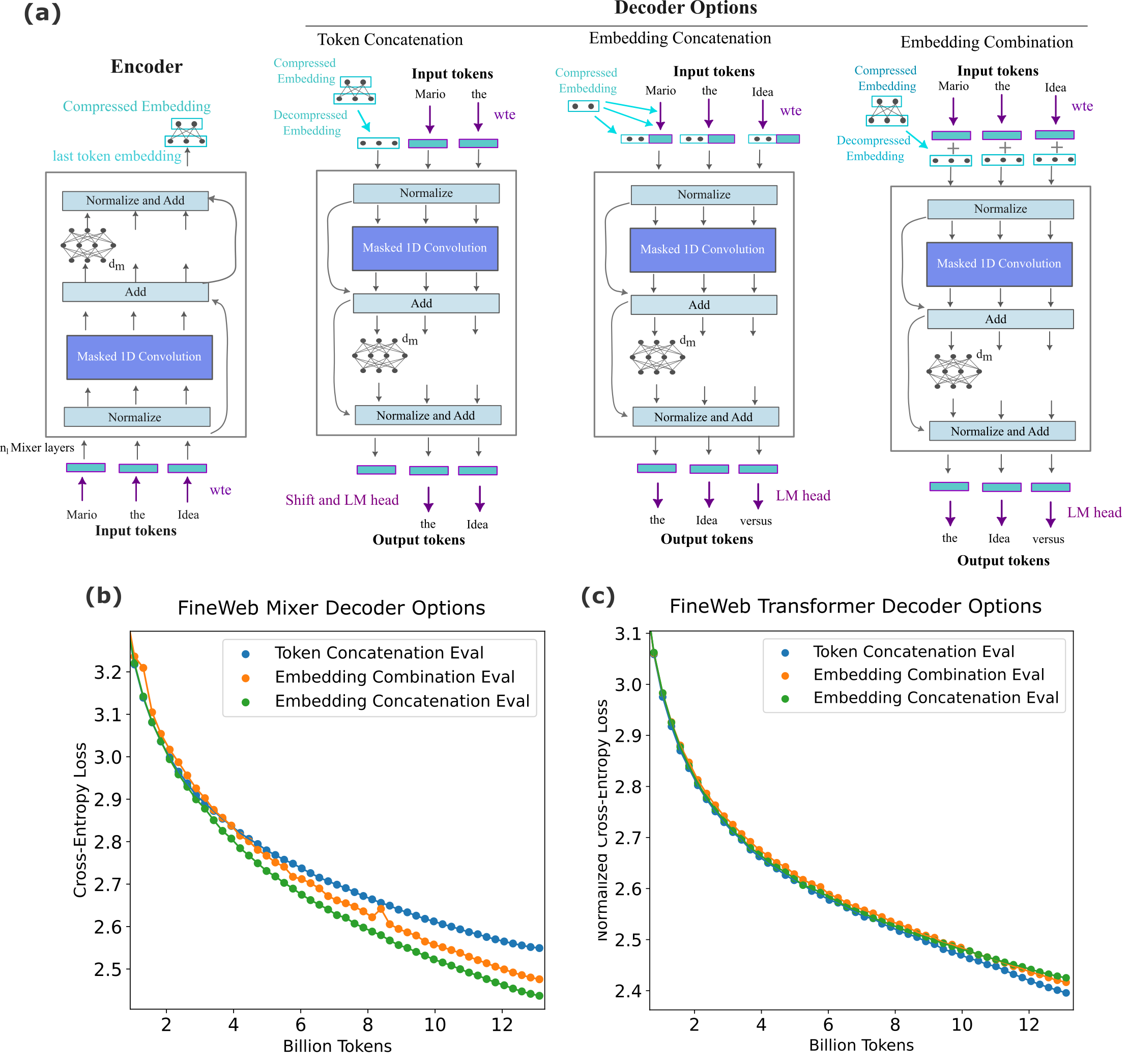}
        \caption{Entropy Estimation model embedding introduction methods and training efficiency.}
        \label{fig3}
    \end{figure}

    A representation of the prediction of one token when using token concatenation is given in Equation \ref{eq5}, where $O_e$ signifies the output of the encoder and $O_d$ is the output of the decoder and $E(x_n)$ the output of the model at token index $n$ and $a \circ b$ denotes the concatenation of vectors $a, b$.

    \begin{equation}
        E(x_{n}) = O_d \left( O_e(x, \theta_e) \circ x_{:n-1}, \theta_d \right)
    \label{eq5}
    \end{equation}

    For all embedding-augmented causal language models, we use a smaller encoder, specifically $d_{enc} = d_{dec} /2$ for transformers and $d_{enc} = d_{dec} / 4$ for mixers such that there is a relatively small throughput penalty for the introduction of this encoder, which on our real-world training runs was typically between 10 and 20 \% depending on the model size and GPU architecture used. We ignore this small constant factor when comparing training efficiencies.
    
\section{Entropy Estimation Models are More Efficiently Trainable than Causal Models} \label{eemsection}

    Causal language model training may be considered a difficult task: we seek to train a model to predict each next token accurately regardless of whether or not that token is actually predictable. We hypothesize that both token prediction as well as predicting which tokens are themselves predictable are difficult tasks, and that disentangling the question of how predictable a token is with the actual prediction of each next token would result in more effective language modeling. One way to accomplish this is to form a (compressed) embedding over the entire input (which we call a `global' encoding to emphasize that it is not causal) and incorporate this into a causal decoder, with the goal of capturing predictability as well as identifying unpredictable tokens in the embedding and the predictable tokens in the decoder. 
    
    From an architecture perspective, swapping out an autoencoder's non-causal decoder for a causal decoder gives two significant benefits: for each sample of size $n_{ctx}$ tokens, the model now only is tasked with predicting one token per forward pass rather than all $n_{ctx}$ tokens at once. Furthermore, tokens $t_0, t_1, ..., t_{n-1}$ are directly supplied to the decoder in addition to an embedding, such that the decoder no longer has to extract information for all these tokens from the embedding.

    \subsection{Entropy Estimation Model Training Characteristics}

    Investigating the size of model required in order to make effective use of a fixed-size global embedding ($d_e=64, n_{ctx}=1024$ unless otherwise noted), we find that there is a large increase in training efficiency up to a certain model size and then a plateau relative to a similarly sized causal model but that the relative increase in training efficiency itself with sample number grows without bound in our experiments (Figure \ref{figs5}). This notably does not imply that a very large entropy estimation model would be inefficient, merely that one cannot expect further proportional increases in per-step loss of that model to a similarly sized causal one. We conjecture that this plateau likely results from inefficiencies in the compression projections rather than the transformer or mixer modules themselves, but leave the investigation of this question to future work. We do not observe increases in entropy estimation model training efficiency upon increases of $k$ or head number, either in the encoder alone or for the entire model (Figure \ref{figs6}).

    We find that both mixer-based and transformer-based entropy estimation models exhibit better training efficiency than their respective causal counterparts, and do indeed exhibit the superior sample scaling properties previously observed in autoencoders, as well as an exponential increase in the loss gap between causal model and entropy estimation model compared to the loss achieved by the causal model (Figures \ref{fig4} and \ref{figs5}). Assuming that these findings translate to larger models trained with more compute and data, this implies that entropy estimation models allow for more feasible estimation of language entropy than is possible using causal models. 

    \begin{figure}[h]
        \centering
        \includegraphics[width=0.99\textwidth]{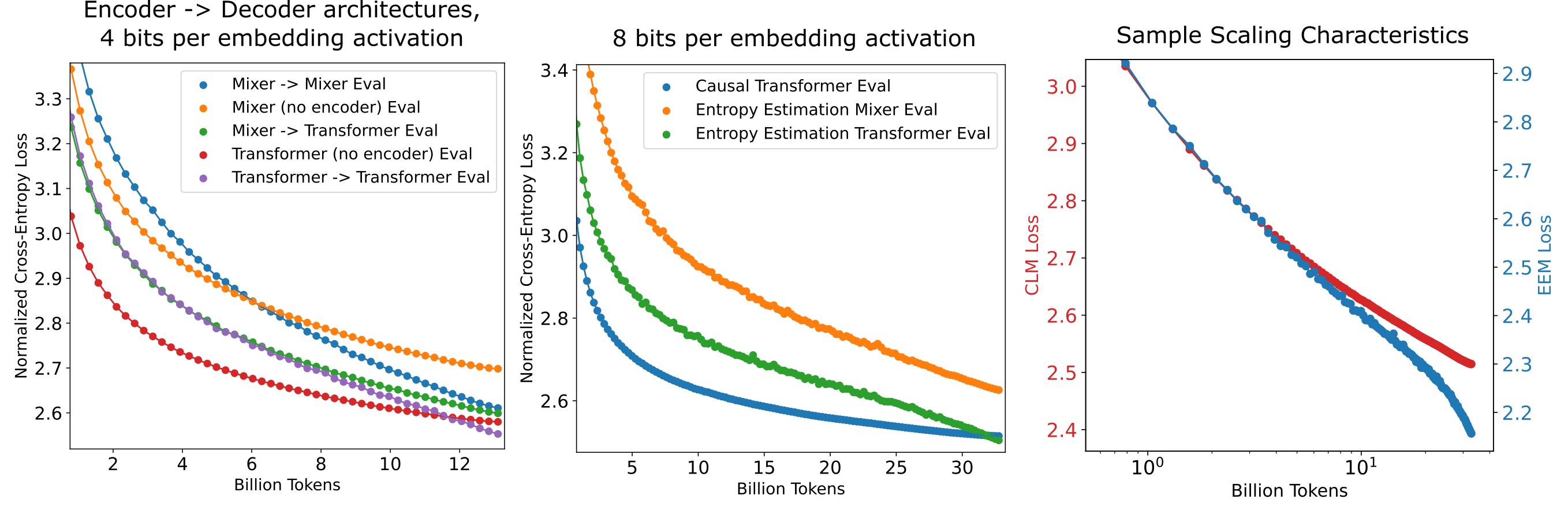}
        \caption{Entropy estimation model compression scaling with tokens trained. Normalization performed according to Equation \ref{eq32}.}
        \label{fig4}
    \end{figure}

    \subsection{Full-context loss scaling}
    
    It should be noted that we train on single documents per sample such that not all samples contain the full context window of nonpad tokens. Evaluating entropy estimation and causal language models on full-context samples only shows that both models perform somewhat worse on these longer prompts but that this difference is nearly constant across training scales as shown in Table \ref{table1}. From these results, we can predict that the entropy estimation model will surpass the causal language model's normalized loss for eight bits per parameter (which equates to $\Bbb L_a = 0.347$ for this setting) at around 700k steps (46 billion tokens) and with less confidence say that least 300 billion tokens would be necessary to reduce the loss of this entropy estimation model to near zero.

        \begin{center}
    \begin{table}[H]
    \begin{center}
    \renewcommand{\arraystretch}{1.2}
    \begin{tabular}{||l c c c c c||} 
     \hline
      Dataset & EEM (13.1 BT) & EEM (32.8 BT) & CLM (13.1 BT) & CLM (32.8 BT) &\\ 
     \hline
      All Samples & 2.379 & 2.157 & 2.580 & 2.515 & \\
     \hline
      Full Context & 2.502 & 2.297 & 2.612 & 2.549 &\\ 
     \hline 
    \end{tabular}
    \end{center}
    \vspace{0.1cm}
    \caption{Losses per transformer-based entropy estimation model or causal language model with the given training samples (in billions of tokens).}
    \label{table1}
    \end{table}
    \end{center}

\section{Noise Injection for Entropy Estimation Model Quantization-Aware Training}

    In the last section it was assumed that four or eight bits per embedding activation would be sufficient, but those models were trained using FP16/FP32 mixed precision \citep{micikevicius2018mixedprecisiontraining}. We therefore investigated whether or not such an assumption was valid, and note that for unmodified models this is not strictly the case: there are increases in evaluation loss when the embedding's activations are quantized via casting to 8-bit types, and substantially less increase in loss when using a more sophisticated data-aware quantization \citep{dettmers2022llmint88bitmatrixmultiplication} as shown in Table \ref{table2}. By measuring the sensitivity of each layer in the entropy estimation model to quantization, we find that the compressed embedding is most sensitive to activation quantization, which is expected given the information passing through this layer (Figure \ref{fig5}).

    \begin{center}
    \begin{table}[H]
    \begin{center}
    \renewcommand{\arraystretch}{1.2}
    \begin{tabular}{||l c c c c c||} 
     \hline
      Model & Float16 (E5M10) & BNB Int8() & Float8 (E4M3fn) & Float8 (E5M2) &\\ 
     \hline
      FP16/FP32 Trained & 2.402 & 2.441 & 2.982 & 4.058 & \\
     \hline
      QAT, $2^{-4}$ noise in embedding & 2.424 & 2.437 & 2.672 & 3.253 &\\ 
     \hline
      QAT, $2^{-2}$ noise pre-embedding & 2.420 & 2.423 & 2.469 & 2.607 &\\
     \hline 
    \end{tabular}
    \end{center}
    \vspace{0.1cm}
    \caption{Evaluation Cross-Entropy Losses for models with the denoted training methods, when quantization is applied to the compressed embedding.}
    \label{table2}
    \end{table}
    \end{center}

    Observing the distribution of activations in the compressed embedding (Figure \ref{figs7} (a)), we hypothesized that using quantization-aware training would allow for near-lossless post-training quantization to 8 bits. We sought a quantization-aware training method that was both simple to implement and hardware-agnostic, as some models were trained on hardware that does not support native 8-bit floats recently introduced in \citep{micikevicius2022fp8formatsdeeplearning}. Inspired by studies estimating the number of bits required per parameter by injecting noise into weights after training and observing the resulting decrease in model accuracy \citep{rumelhart1986parallel, sejnowski1987parallel}, we designed a quantization-aware training (QAT) procedure by which uniform noise is injected into the compressed embedding layer's activations according to Equation (\ref{eq7}), where we range $q=2^{-4}$ to $q=2^{-1}$.
    
    \begin{equation}
    O_{up} = W_{up} \left( W_{down}x + \mathcal{U}_e(-q, q) \right)
    \label{eq7}
    \end{equation}

    \subsection{Noise Injection leads to Activation and Weight Quantization Insensitivity}

    After training models using noise injection, we observe the loss on hold-out evaluation data when the embedding is quantized to eight bits per activation either using naive bit casting to one of the formats introduced in \citep{micikevicius2022fp8formatsdeeplearning} or else a much more sophisticated approach of performing data-aware quantization introduced in \citep{dettmers2022llmint88bitmatrixmultiplication}. The latter method involves quantization of weight matrices as well as activations of a particular layer, so we apply this quantization to an identity transformation introduced into the compressed layer in order to separate the effects of weight and activation quantizations. We find that training with injection of noise in the compressed embedding layer results in a small increase in evaluation loss, but that this imparts a decrease in loss upon quantization compared to models trained without noise injection (Table \ref{table2}). We find that injecting noise before the compressed embedding results in the smallest increase in quantized versus non-quantized loss (Table \ref{table2}).

    \begin{figure}[h]
        \centering
        \includegraphics[width=0.99\textwidth]{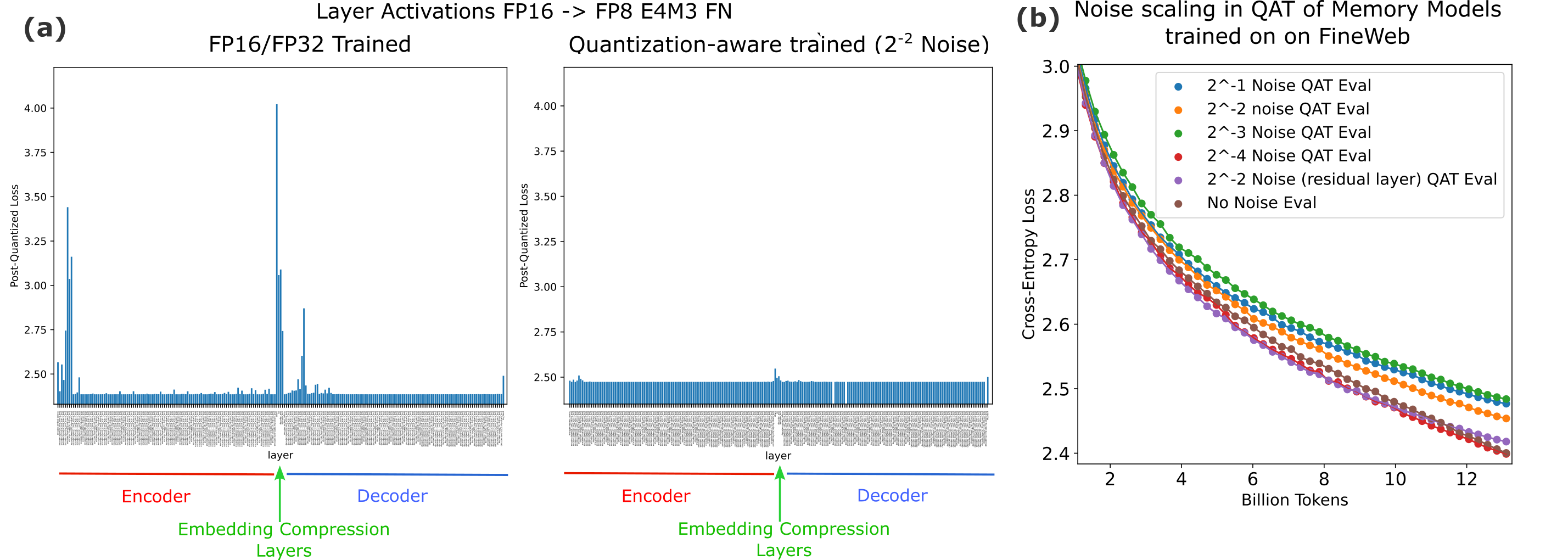}
        \caption{Quantization-aware training via noise injection imparts quantization insensitivity to all layer activations and results in a minimal decrease in training efficiency.}
        \label{fig5}
    \end{figure}

    Quantization-aware training is not particularly useful for entropy estimation models if it results in significant decreases in training efficiency relative to unquantized models, as the primary benefit of this type of model compared to causal models is its training efficiency. We characterize this in Figure \ref{fig5}, and find that there is very little no decrease in training efficiency when models are trained with noise necessary for near-lossless 8-bit embedding compression.
    
\section{Per-Token Conditional Entropy Estimation}

    The superior training properties of entropy estimation models relative to causal models allows for more efficient estimation of source entropy. At the level of the text corpus, this allows one to more accurately predict how much compute should be applied to a given corpus during causal model training: if a model were to approach the entropy estimation amount, further training would not be expected to yield much generalizable benefit. Besides an estimation of entropy in a corpus of text or other sequential data, we can also use entropy estimation models to obtain entropy estimates for individual tokens. In this section we detail methods for computing per-token entropy using encoder-decoder entropy estimation models, and show that fast proxy methods for doing so are poor approximators, but that models may be trained to become excellent approximators of these values.

    \subsection{Token Entropy Calculation}

    Obtaining entropy estimates from a causal model requires only the observation of unreduced loss, where the model's entropy estimation for token $n$ is simply the unreduced Cross-Entropy loss between the model's output given previous tokens $O(t_{:n-1}, \theta)$ and that token. As the entropy estimation model's encoder is globally attentive, we cannot perform a similar decomposition in the sequence dimension for these models. Conceptually this is because although one may similarly observe the unreduced loss values at each token position given an entropy estimation model, these are unreliable estimates of token entropy as the encoder's embedding contributes an unknown amount of information to the prediction of this token. Decomposing the amount of information per token is a difficult task as the embedding's relationship to each model output is highly nonlinear.

    Per-token entropy estimations may instead be obtained as follows: given two encoder-decoder entropy estimation models $\theta_1, \theta_2$ where $\theta_1$ has a context window of size $N-1$ and $\theta_2$ has a context window of size $N$, and for simplicity we assume that the embeddings of these models (in bytes) are the same size, $\vert e_1 \vert = \vert e_2 \vert$ although this is certainly not a necessary condition. We can then compute the entropy of the token at position $N+1$ given the tokens at position $N$ using these models by applying the chain rule of entropy as shown in Equation \ref{eq3}, where the token entropy at each position may be obtained by a sliding window approach as shown in Figure \ref{fig6}. Note that to differentiate cross-entropy from joint entropy, we denote cross-entropy of $a, b$ as $H_b(a)$.

    \begin{equation}
        \begin{split}
       H(t_{N} \vert t_{0}, t_{1}, ..., t_{N-1}) & = H(t_0, t_1, ..., t_N) - H(t_0, t_1, ..., t_{N-1}) \\
        & = H_{(t_1, t_2, ..., t_{N})} O((t_0, t_1, ..., t_{N-1},  \theta_2))  - H_{(t_1, t_2, ..., t_{N-1})} O((t_0, t_1, ..., t_{N-2}, \theta_1) \\
        & = \frac{\vert e \vert}{L_t} \ln (2) + \sum_{i=0}^N \Bbb L(O(t_{0}, t_{1}, ..., t_{i-1}, \theta_2), t_i) - \left( \frac{\vert e \vert}{L_t} \ln (2)  + \sum_{i=0}^{N-1}\Bbb L(O(t_{0}, t_{1}, ..., t_{i-1}, \theta_1), t_i) \right) \\
        & = \sum_{i=0}^N \Bbb L(O(t_{0}, t_{1}, ..., t_{i-1}, \theta_2), t_i) -  \sum_{i=0}^{N-1}\Bbb L(O(t_{0}, t_{1}, ..., t_{i-1}, \theta_1), t_i)
        \end{split}
    \label{eq3}
    \end{equation}

    For efficiency at the expense of a certain amount of accuracy, one model can be used instead of two, where a pad token is inserted into to the $t_{-1}$ index and Equation \ref{eq4} is applied.  For this to be an accurate estimate, it is assumed that the model is trained on data that is left padding and that the encoder and encoder were trained on these padded inputs. We find that per-token entropy estimates using this method correlate reasonably well with per-token entropy estimates obtained from a causal language model that achieves nearly equivalent total compression (Figure \ref{fig6}). 

    \begin{equation}
    H(t_{N} \vert t_{:N-1}) = H_{t_{1:N}} (O(t_{0:N-1}, \theta)) - H_{ t_{0: N-1}}\left( O(t_{-1:N-2}, \theta) \right)
    \label{eq4}
    \end{equation}

    \begin{figure}[h]
        \centering
        \includegraphics[width=0.99\textwidth]{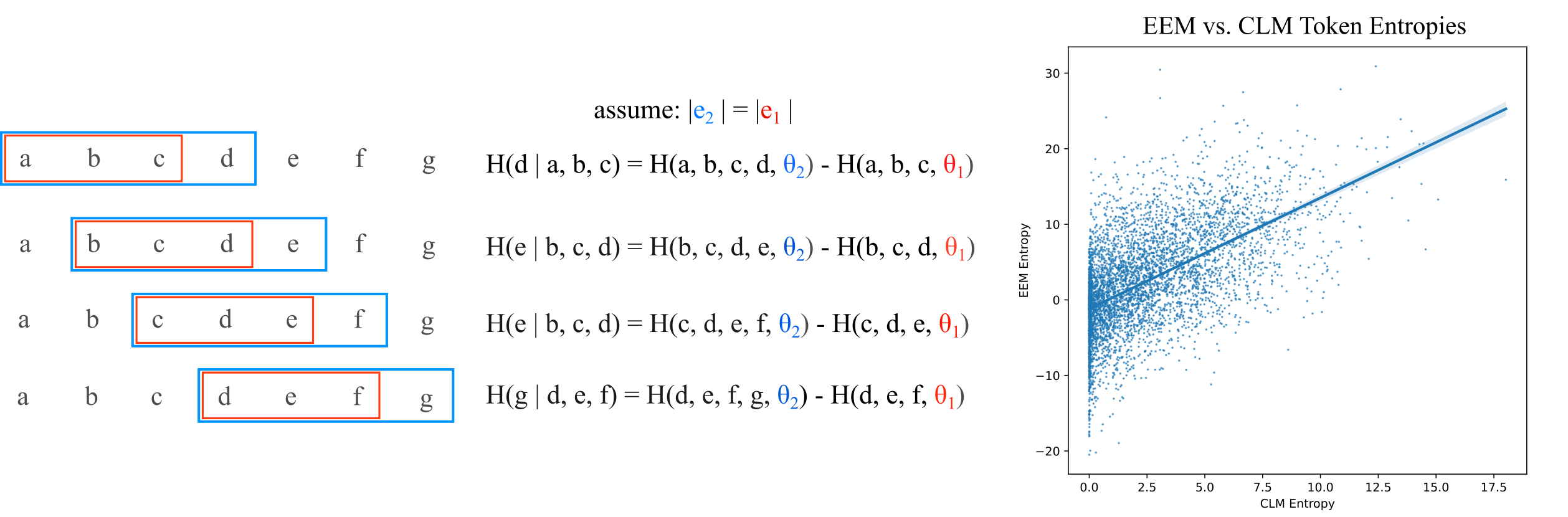}
        \caption{Two-model per-token entropy calculation (left) and single-model entropy estimation correlation with causal language model cross-entropies per token (right, $y = 1.469x - 1.1879$, $R^2=0.329$).}
        \label{fig6}
    \end{figure}

    \subsection{Fast Estimation Methods don't Correlate with Causal Entropy Estimations}

    The use of Equation \ref{eq3} or Equation \ref{eq4} with an entropy estimation model give accurate (assuming the stated conditions are met) but slow estimates relative to those obtained by a causal language model: only one token's entropy estimate is obtained per forward pass, compared to the $n_{ctx}$ token estimates obtained in one forward pass of a causal model. We therefore investigated methods to approximate entropy estimates that could be performed such that all token's entropies are calculated in one or two forward passes. We approximate entropy estimation models' outputs by observing the change in the model output upon occlusion of the embedding as detailed in Section \ref{proxysection}, but find that this does not result in accurate entropy estimates (Supplementary Table \ref{tables1}). Occlusion-based methods yield inconsistent average per-index token entropies per index relative to causal model entropy estimates as well (Supplementary Figure \ref{figs8}).

    \subsection{Second Order Entropy Estimation Models}

    We conjectured that a more direct, trainable method for estimating token entropies exists where a model $\theta_a$ could learn to predict conditional entropies that were first computed by another model $\theta_b$. We call $\theta_a$ a `second order' entropy estimation model because its input target labels are themselves the outputs of another model. A second order entropy estimation model can be implemented using language modeling architectures with a regression head replacing a language modeling (token prediction) one. We use causal backbones for entropy estimation to match token prediction modeling priors. 
    
    Reasoning that it would be easier for a model to predict the uncertainty of a token that it knows compared to the uncertainty of a token that it does not, we investigated the prediction of entropies in datasets where the output is shifted in the token dimension (similarly to what is done for causal language modeling) as well as when the output is not shifted. Empirically confirming that this is indeed the case, we find that per-token entropy estimates may be learned with relatively high precision (Figure \ref{fig8}).

    \begin{figure}[h]
        \centering
        \includegraphics[width=0.99\textwidth]{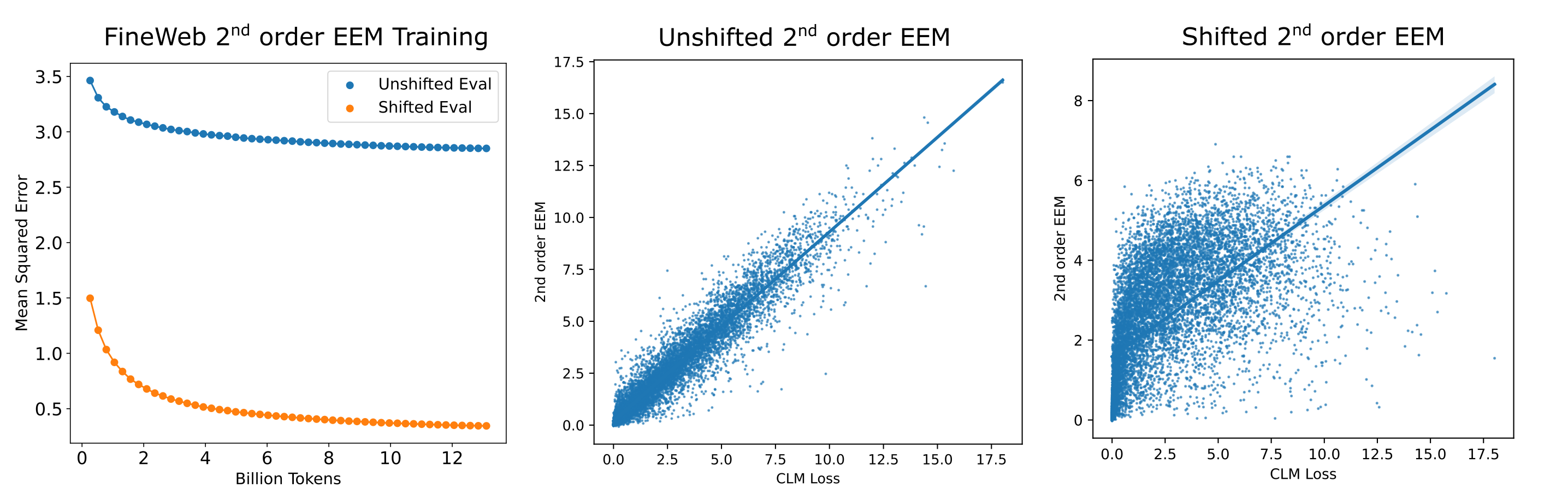}
        \caption{Accurate entropy estimation in a token-shifted versus unshifted second-order transformer entropy estimation model. For the unshifted model the OLS equation is $y = 0.911x + 0.198$ with an $R^2=0.919$, for the shifted model $y = 0.378x + 1.588$ with $R^2=0.369$.}
        \label{fig8}
    \end{figure}

    The substantial difference between second order EEM training efficiencies when the identity of the next token is known (unshifted) and when it is not known (shifted) supports our initial hypothesis stated in Section \ref{eemsection}, that causally predicting which tokens are and are not predictable is in itself a difficult task even when we ignore the task of actual token prediction. Causal models of course do not have the identity of each next token available during training, and therefore we can expect their ability to estimate the predictability of each token to be at least as inaccurate as the shifted second-order EEM (as causal models must also learn to predict tokens at the same time). Entropy estimation model decoders, on the other hand, are given at least an approximation of the identity of each next token by the compressed embedding and therefore can be trained in a somewhat similar manner to the unshifted second-order entropy estimation model, which is clearly more efficient.
    
\section{Causal Models Generalize Better when Trained with Entropy Estimations}

    In some sense it is clear that distillation of an entropy estimation model must occur via some introduction of the per-token entropies into the training of another model.  The natural question to ask therefore is whether this per-token entropy information is beneficial in some way for causal language modeling. Although there are undoubtedly other uses, we focus here on the relationship between this information and model generalization.

    \subsection{Entropy-informed Training yields Ideal Generalization}
    
    We first seek to determine an answer to the question of whether or not training a model to minimize cross-entropy below the the training dataset's entropy is detrimental to generalization.  Another way of saying this is as follows: training a model past its dataset's entropy may either result in an increase, no change, or a decrease in the hold-out evaluation dataset's entropy.  There is intuition for both `no change' and an `increase' scenarios: on one hand if a model learns to approximate its training data beyond what is possible considering that data's entropy, the model must by definition not generalize to hold-out test data or the dataset's entropy would be the new lower estimate provided by the model. Thus anything learned by a model to decrease its loss beyond the dataset's intrinsic entropy should in principle result a decrease in generalization. At the same time, it seems reasonable to expect that there is some function capable of minimizing the objective function on the training data without affecting the hold-out test set.

    We claim that training a model to minimize its cross-entropy loss with training data towards but not beyond the entropy of that data results in the ideal generalization, defining generalization as the minimum value of the cross-entropy loss of the model on the hold-out test data. We formalize this statement and prove it in Theorem 1, as it follows directly from Gibbs' inequality and the independence and identical distribution of training and test data. This effectively rules out the possibility that training a model's loss below the entropy of the training data does not affect test set performance. 

    \topsep=12pt
    \newtheorem{theorem}{Theorem}

    \begin{theorem}
        Generalization decreases when models are trained to minimize their losses below the entropy of the training dataset.
    \end{theorem}

    \begin{proof}

    We assume that some random variable $X$ exists and contains independent and identically distributed disjoint subsets, $X_{train}$ and $X_{text}$, where $\{x' \sim X_{train}\} \; \bigcup \; \{x'' \sim X_{test}\} = \{x \sim X\}$ and $\{x' \sim X_{train}\} \; \bigcap \; \{x'' \sim X_{test}\} = \emptyset$. For notational clarity, we substitute $S = X_{train}$ and $T = X_{test}$ and use $H_B(A)$ to denote the cross-entropy between random variables $A$ and $B$, and $H(A, B)$ to denote the joint entropy of $A$ and $B$. 
    
    We assume that the entropy of the random variable is nonzero, $H(X) > 0$, and by the independence of $S, T$ we have $H(X) = H(S, T) = H(S) + H(T)$.  We also assume that model losses for samples drawn from these random variables, $s \sim S, t \sim T$ are cross-entropies on causal predictions, where we denote the cross-entropy of a model's output on an arbitrarily large number of samples of a random variable as $H_S(O(S, \theta)) = \sum_n \sum_i\Bbb L(O(s_{[:i-1]}, \theta), s_i)_n = \sum_n H_s(O(s_n, \theta), s_n) :  s_n \sim S$ and the same for $T$.

    From Gibbs' inequality, we know that $H_X(O(X, \theta)) \geq H(X) \; \forall  X, \theta$. We decompose $X$ according to the independence of its subsets $S, T$ as follows:
    
    \begin{equation}
        \begin{split}
         H_X(O(X, \theta)) = & H_S(O(S, \theta)) + H_T(O(T, \theta)) \geq H(X) = H(S) + H(T) \implies \\
         & H_S(O(S, \theta)) - H(S) + \left( H_T(O(T, \theta)) - H(T)  \right) \geq 0 
         \end{split}
         \label{eq15}
    \end{equation}

    Therefore as $\theta$ is modified such that $H_S(O(S, \theta)) - H(S)$ decreases below $0$, $H_T(O(T, \theta)) - H(T)$ increases so that the Inequality \ref{eq15} holds. We assumed that $H(T)$ is fixed, so necessarily $H_T(O(T, \theta))$ increases as $H_S(O(S, \theta))$ decreases below $H(S)$. 
    
    In fact, due to the identical distribution of $S, T$ we know that $H(S) = H(T)$ and can put a lower bound on the decrease of generalization as shown in Equation \ref{eq16}.

    \begin{equation}
        H_S(O(S, \theta) - H(S) \geq - \left( H_T(O(T, \theta)) - H(S) \right)
        \label{eq16}
    \end{equation}
    
    We conclude by restating Equation \ref{eq16} in words: as it was assumed $H_S(O(S, \theta)) < H(S)$ making the left side of the inequality negative, the difference between the model's cross-entropy loss and the training dataset's entropy is at least as large as the difference between the model's loss on the test set and that entropy.
    
    \end{proof}

    \textit{Remark.}
    For finite datasets, a simple proof of Theorem 1 may be sketched as follows: the entropy of the dataset determines its minimal compression amount (say $q$ bits), such that training a model to compress a subset of that dataset to a higher degree than this minimal amount necessarily results in a decrease in compression of the rest of that dataset so that the total compression remains above the minimal value. More precisely, the number of bits required for $s \in S$, $C(s)$, plus the number of bits required for $t$, $C(t)$ is $C(s) + C(t) = q$. As $q$ is fixed, any model achieving compression $O(s, \theta) < C(s)$ must also yield $O(t, \theta) > C(t)$ or otherwise $O(s, \theta) + O(t, \theta) < C(s) + C(t) = q$, an impossibility. 
    
    It should also be noted that the question of whether minimization of training set loss all the way to the entropy of that dataset is necessary for ideal generalization, meaning that we change $\theta$ such that $H_S(O(S, \theta)) \to H(S) \implies \min_{\theta} H_T(O(T, \theta))$, is not addressed by this approach, and we argue this is fundamentally dataset-dependent. The reason for this is that in the limiting case where $S$ and $T$ grow to infinity, $H_S(O(S, \theta)) \to H(S) \implies H_T(O(T, \theta)) \to H(T)$ and it is indeed necessary for the model to reach the training dataset's entropy to minimize the test set's loss (also to the entropy of the dataset). There is no guarantee that this is the case for finite datasets, however, and it is highly unlikely that in the other extreme case where $|\{s \in S\}| = |\{t \in T\}| = 1$ the implication would hold. Taking $H(S) = H(t) = 0$ with this case means that the model would have to correctly guess the function that determines $S, T$ out of all functions that determine $S$, which may or may not be probable and depends on the space of functions that can approximate $S$ and $T$ as well as the nature of the model $\theta$.

    \subsection{Empirically Superior Generalization with Entropy-Informed Training}

    From Theorem 1 we are guaranteed that training a model to exceed (by which we mean decrease below) a dataset's entropy will result in overfitting, but if the model does not exceed the dataset's entropy then we have no guarantees that overfitting will or will not occur. We are left with the following question: does training a model to approach but not exceed a dataset's entropy value result in greater generalization than if we train the model to reduce its cross-entropy loss without regard to the dataset's intrinsic entropy? Another way of thinking about this is to notice that many current approaches to increasing model generalization (early stopping, dropout, $L^2$ regularization etc.) may prevent the model from exceeding its dataset's entropy almost inadvertently, so is it actually more beneficial to use entropy estimation data while training rather than simply relying on existing generalization-boosting methods? We present an argument that it is better to use entropy information, before showing that empirically this is indeed the case. 
    
    Decomposing entropies per dataset into entropies per conditional token is required to make use of Theorem 1 for causal model training. Now consider what occurs when we train a model to approximate the entropy estimates $e_i$ from another model, which we accomplish by rescaling the loss used for model backpropegation according to Equation \ref{eq17}. The resulting model learns to approximate a function that minimizes the loss in Equation $\ref{eq17}$ for each token $i$ in every sample of $x$, instead of a function that results in total minimization of the total cross-entropy loss (Equation \ref{eq18}) without respect to loss per token. 

    \begin{equation}
        \Bbb L(O(x, \theta), x) = \sum_i || \Bbb L(O(x_{:i-1}, \theta), x_i) - e_i ||_1
        \label{eq17}
    \end{equation}

    \begin{equation}
        \Bbb L(O(x, \theta), x) = \sum_i \Bbb L(O(x_{:i-1}, \theta), x_i)
        \label{eq18}
    \end{equation}

    We argue that training a model $\theta_a$ to approximate the per-token entropies of a model with a known goodness of generalization $\theta_b$ should result in better generalization of $\theta_a$, as long as per-token entropy imparts some generalization characteristics onto this model. To test this, we initialize a model and training dataset combination that is prone to severe overfitting: a model that has more parameters (around 75 million) than training data points (around 50 million tokens) and proceed to train for many epochs. Applying Equation (\ref{eq17}) to this setting, we find substantially better generalization both in terms of lowest hold-out test loss achieved and the stability of this loss in terms of the number of training epochs in which the model approximates this value compared to an identical experimental setting applying Equation (\ref{eq18}), with early stopping or early stopping and dropout ($p=0.1$ applied to the attention values) (Figure \ref{fig7}). We note that the use of entropy estimates is composable with other regularizers such as dropout, which combined yield the greatest generalization we measure (Figure \ref{fig7}, Table \ref{table3}). All models use weight decay intrinsic to the AdamW optimizer, and we do not test $L^2$ weight regularization due to its relatively ineffective application for adaptive optimizers \citep{loshchilov2019decoupledweightdecayregularization}.

    \begin{figure}[h]
        \centering
        \includegraphics[width=0.99\textwidth]{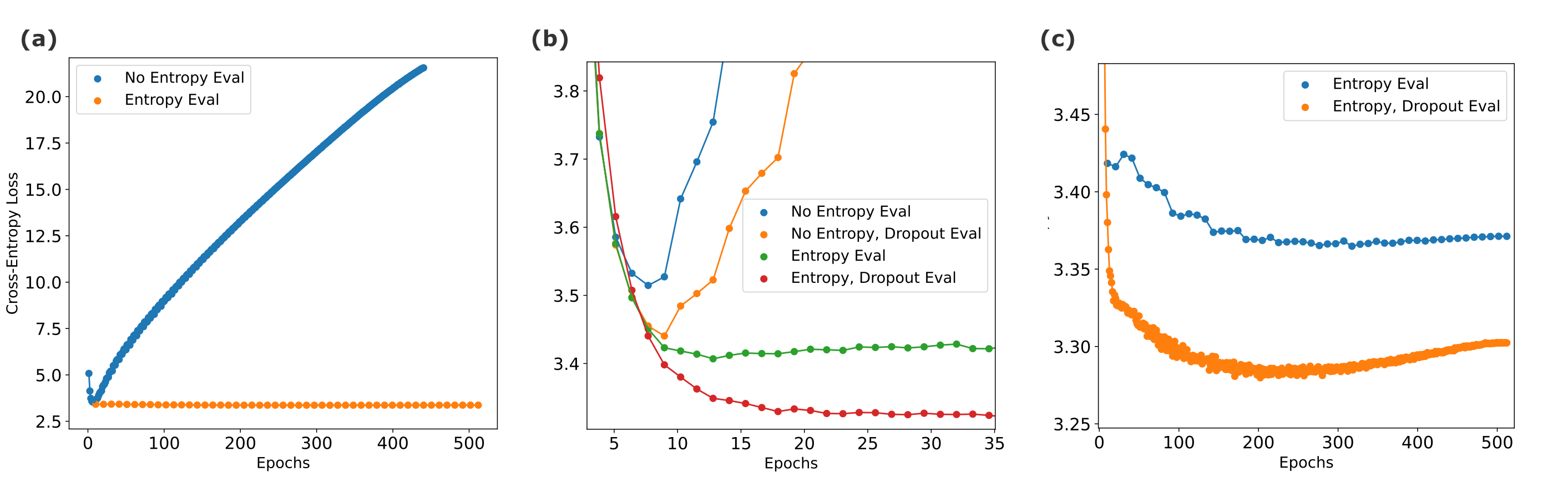}
        \caption{Single model per-token entropy calculation and correlation with CLM token entropies. All models are the same architecture used for transformer-based causal modeling in other figures ($d_m = 512, n_l=16, n_{ctx}=1024$) and are trained on the first 50k samples on the FineWeb-edu dataset, and losses are computed on evaluation (hold-out) data. (a) Unbounded overfitting for transformers trained on the small dataset without but not with loss rescaling according to Equation \ref{eq17}. (b) and (c) Overfitting evaluation with and without $p=0.1$ dropout in attention, with and without entropy estimation loss rescaling.}
        \label{fig7}
    \end{figure}

    It should be noted that the same argument can extended to the use of many entropy estimates rather than only one. What we are in effect attempting to do is to introduce entropy information in order to better approximate our model's trajectory of parameter states during training $\theta_0, \theta_1, ..., \theta_N$ with another model's trajectory $\phi_0, \phi_1, ..., \phi_N$, where we know that $\phi_N$ generalizes well (and therefore that $\phi_1, \phi_2, ..., \phi_{N-1}$ do too for most models). Training $\theta$ to match many sequential entropy estimates of $\phi$ rather than only one is a natural way to more closely align these two trajectories, but we leave this to future work.

    \begin{center}
    \begin{table}[H]
    \begin{center}
    \renewcommand{\arraystretch}{1.2}
    \begin{tabular}{||l c c c c c||} 
     \hline
      Model & No Entropy & No Entropy, Dropout & Entropy & Entropy, Dropout &\\ 
     \hline
      Transformer CLM & 3.515 & 3.440 & 3.364 & 3.280 & \\
     \hline
     \hline 
    \end{tabular}
    \end{center}
    \vspace{0.1cm}
    \caption{Evaluation cross-entropy loss for small-subset (50k sample) FineWeb, all with early stopping.}
    \label{table3}
    \end{table}
    \end{center}

\section{Conclusion}

    \subsection{Limitations of this work}

    One of the main limitations of the encoder-decoder entropy estimation model is that it requires $n_{ctx}$ times the compute for entropy estimation as a strictly causal model. Furthermore, the entropy estimates these models yield are inaccurate for tokens near the start of sequences: we find that the first 100 or so tokens are poorly estimated using our $n_{ctx}=1024$ model. We note that neither of these limitations are insurmountable: firstly, because models trained with smaller context windows would be able to provide accurate per-token entropy estimates for tokens near the start of a sequence, and in the second case because we show empirically that second order entropy estimation models are efficiently trainable.

    Limited in compute, this work does not directly test some hypotheses put forward and resorts in some cases to observing the limiting behaviors of smaller experiments to infer what would happen for larger ones. Whether these approaches scale as effectively as they are predicted to do so for very large compute is an open question, one that we hope will be picked up by the field and brought to a more satisfactory conclusion in the future.

    \subsection{Generalization and Entropy}

    Methods for increasing the generalization of machine learning model training abound, but most approaches may be thought of as a conversion of unconstrained optimization to a constrained optimization problem. To illustrate: optimization with $L^2$ weight decay constrains the norms of the weight `vectors', early stopping constrains the norm of the difference in the weight of the final compared to the initial model state (assuming a fixed or decaying update size), and dropout constrains the combinatorial paths by which information can pass through a model. The training of models using entropy estimates is a departure from these methods: training a model to match an entropy estimate is not strictly speaking constrained optimization as there are no priors placed on the model or optimizer, rather instead the objective function itself is re-parametrized according to both the model's current loss and the entropy estimates. 

    It may be wondered whether having an efficiently trainable entropy estimator would be useful if compute were to become one day so plentiful that a large model could be trained on all text data in the world, regardless of the likelihood of this scenario. We argue that it is: entropy estimation models may be applied to derived text datasets (proofs etc) as well as other data modalities (audio, visual) to gain more accurate entropy estimates even in this scenario, even if models themselves converge on one representation \citep{huh2024platonicrepresentationhypothesis}. 

    \subsection{Entropy as an Optimization Rate Limiter}

    Per-token entropy estimates may also be useful to accelerate the training convergence of causal language models. The idea that noisy input information should be somehow filtered during gradient-based optimization procedures is not new, and in the context of language modeling, data filtering most often occurs at the level of the source document.  It is not inaccurate to view the process of training a causal model on entropy-labeled data as a form of data filtering, but in this case the filtering occurs at the level of the token rather than the document. Concepts of information filtering have found their way into model architectures and optimizers themselves: attention mechanism ubiquitous in transformers today was originally introduced in order to effectively filter information from many tokens \citep{bahdanau2016neuralmachinetranslationjointly}, and the most commonly used optimizers today (Adam and AdamW) for language modeling were in a large part designed to deal with stochastic input data by using momentum estimates \citep{kingma2017}. Entropy estimation-based training can also be viewed as performing an analogous noise-filtering method that Adam optimizers do at the parameter level (using second adaptive momentum) but at the level of the token. We leave the investigation of entropy and optimization efficiency to future work.

\bibliographystyle{unsrtnat}
\bibliography{references}  

\begin{thebibliography}{35}
\providecommand{\natexlab}[1]{#1}
\providecommand{\url}[1]{\texttt{#1}}
\expandafter\ifx\csname urlstyle\endcsname\relax
  \providecommand{\doi}[1]{doi: #1}\else
  \providecommand{\doi}{doi: \begingroup \urlstyle{rm}\Url}\fi

\bibitem[Shannon(1948)]{shannon1948mathematical}
Claude~E Shannon.
\newblock A mathematical theory of communication.
\newblock \emph{The Bell system technical journal}, 27\penalty0 (3):\penalty0 379--423, 1948.

\bibitem[Wiener(2019)]{wiener2019cybernetics}
Norbert Wiener.
\newblock \emph{Cybernetics or Control and Communication in the Animal and the Machine}.
\newblock MIT press, 2019.

\bibitem[Shannon(1951)]{shannon1951prediction}
Claude~E Shannon.
\newblock Prediction and entropy of printed english.
\newblock \emph{Bell system technical journal}, 30\penalty0 (1):\penalty0 50--64, 1951.

\bibitem[Bellard(2019)]{bellard2019lossless}
Fabrice Bellard.
\newblock Lossless data compression with neural networks.
\newblock \emph{URL: https://bellard. org/nncp/nncp. pdf}, 2019.

\bibitem[Mahoney(2025)]{mahoneylargetext}
Matt Mahoney.
\newblock Large text compression benchmark, 2025.
\newblock URL \url{https://www.mattmahoney.net/dc/text.html#1072}.

\bibitem[Hutter et~al.(2025)Hutter, Bowery, and Mahoney]{hutterprize}
Hutter, Bowery, and Mahoney.
\newblock 500'000€ prize for compressing human knowledge, 2025.
\newblock URL \url{http://prize.hutter1.net/}.

\bibitem[Knoll(2024)]{knollcmix}
Bryan~C. Knoll.
\newblock Cmix, 2024.
\newblock URL \url{https://www.byronknoll.com/cmix.html}.

\bibitem[Gao et~al.(2020{\natexlab{a}})Gao, Biderman, Black, Golding, Hoppe, Foster, Phang, He, Thite, Nabeshima, et~al.]{gao2020pile}
Leo Gao, Stella Biderman, Sid Black, Laurence Golding, Travis Hoppe, Charles Foster, Jason Phang, Horace He, Anish Thite, Noa Nabeshima, et~al.
\newblock The pile: An 800gb dataset of diverse text for language modeling.
\newblock \emph{arXiv preprint arXiv:2101.00027}, 2020{\natexlab{a}}.

\bibitem[Deepseek(2025)]{deepseekai2025deepseekv3technicalreport}
Deepseek.
\newblock Deepseek-v3 technical report.
\newblock 2025.
\newblock URL \url{https://arxiv.org/abs/2412.19437}.

\bibitem[Gao et~al.(2020{\natexlab{b}})Gao, Biderman, Black, Golding, Hoppe, Foster, Phang, He, Thite, Nabeshima, Presser, and Leahy]{gao2020pile800gbdatasetdiverse}
Leo Gao, Stella Biderman, Sid Black, Laurence Golding, Travis Hoppe, Charles Foster, Jason Phang, Horace He, Anish Thite, Noa Nabeshima, Shawn Presser, and Connor Leahy.
\newblock The pile: An 800gb dataset of diverse text for language modeling, 2020{\natexlab{b}}.
\newblock URL \url{https://arxiv.org/abs/2101.00027}.

\bibitem[Delétang et~al.(2024)Delétang, Ruoss, Duquenne, Catt, Genewein, Mattern, Grau-Moya, Wenliang, Aitchison, Orseau, Hutter, and Veness]{delétang2024languagemodelingcompression}
Grégoire Delétang, Anian Ruoss, Paul-Ambroise Duquenne, Elliot Catt, Tim Genewein, Christopher Mattern, Jordi Grau-Moya, Li~Kevin Wenliang, Matthew Aitchison, Laurent Orseau, Marcus Hutter, and Joel Veness.
\newblock Language modeling is compression.
\newblock 2024.
\newblock URL \url{https://arxiv.org/abs/2309.10668}.

\bibitem[Wei et~al.(2025)Wei, Sun, and Li]{wei2025deepseekocrcontextsopticalcompression}
Haoran Wei, Yaofeng Sun, and Yukun Li.
\newblock Deepseek-ocr: Contexts optical compression.
\newblock 2025.
\newblock URL \url{https://arxiv.org/abs/2510.18234}.

\bibitem[Shalev et~al.(2020)Shalev, Painsky, and Ben-Gal]{shalev2020neuraljointentropyestimation}
Yuval Shalev, Amichai Painsky, and Irad Ben-Gal.
\newblock Neural joint entropy estimation.
\newblock 2020.
\newblock URL \url{https://arxiv.org/abs/2012.11197}.

\bibitem[Hoffmann et~al.(2022)Hoffmann, Borgeaud, Mensch, Buchatskaya, Cai, Rutherford, de~Las~Casas, Hendricks, Welbl, Clark, Hennigan, Noland, Millican, van~den Driessche, Damoc, Guy, Osindero, Simonyan, Elsen, Rae, Vinyals, and Sifre]{hoffmann2022trainingcomputeoptimallargelanguage}
Jordan Hoffmann, Sebastian Borgeaud, Arthur Mensch, Elena Buchatskaya, Trevor Cai, Eliza Rutherford, Diego de~Las~Casas, Lisa~Anne Hendricks, Johannes Welbl, Aidan Clark, Tom Hennigan, Eric Noland, Katie Millican, George van~den Driessche, Bogdan Damoc, Aurelia Guy, Simon Osindero, Karen Simonyan, Erich Elsen, Jack~W. Rae, Oriol Vinyals, and Laurent Sifre.
\newblock Training compute-optimal large language models.
\newblock 2022.
\newblock URL \url{https://arxiv.org/abs/2203.15556}.

\bibitem[Vaswani et~al.(2023)Vaswani, Shazeer, Parmar, Uszkoreit, Jones, Gomez, Kaiser, and Polosukhin]{vaswani2023attentionneed}
Ashish Vaswani, Noam Shazeer, Niki Parmar, Jakob Uszkoreit, Llion Jones, Aidan~N. Gomez, Lukasz Kaiser, and Illia Polosukhin.
\newblock Attention is all you need.
\newblock 2023.
\newblock URL \url{https://arxiv.org/abs/1706.03762}.

\bibitem[Badger(2025)]{badger2025maskedmixerslanguagegeneration}
Benjamin~L. Badger.
\newblock Masked mixers for language generation and retrieval.
\newblock 2025.
\newblock URL \url{https://arxiv.org/abs/2409.01482}.

\bibitem[Penedo et~al.(2024)Penedo, Kydlíček, allal, Lozhkov, Mitchell, Raffel, Werra, and Wolf]{penedo2024finewebdatasetsdecantingweb}
Guilherme Penedo, Hynek Kydlíček, Loubna~Ben allal, Anton Lozhkov, Margaret Mitchell, Colin Raffel, Leandro~Von Werra, and Thomas Wolf.
\newblock The fineweb datasets: Decanting the web for the finest text data at scale.
\newblock 2024.
\newblock URL \url{https://arxiv.org/abs/2406.17557}.

\bibitem[Allal et~al.(2025)Allal, Lozhkov, Bakouch, Blázquez, Penedo, Tunstall, Marafioti, Kydlíček, Lajarín, Srivastav, Lochner, Fahlgren, Nguyen, Fourrier, Burtenshaw, Larcher, Zhao, Zakka, Morlon, Raffel, von Werra, and Wolf]{allal2025smollm2smolgoesbig}
Loubna~Ben Allal, Anton Lozhkov, Elie Bakouch, Gabriel~Martín Blázquez, Guilherme Penedo, Lewis Tunstall, Andrés Marafioti, Hynek Kydlíček, Agustín~Piqueres Lajarín, Vaibhav Srivastav, Joshua Lochner, Caleb Fahlgren, Xuan-Son Nguyen, Clémentine Fourrier, Ben Burtenshaw, Hugo Larcher, Haojun Zhao, Cyril Zakka, Mathieu Morlon, Colin Raffel, Leandro von Werra, and Thomas Wolf.
\newblock Smollm2: When smol goes big -- data-centric training of a small language model.
\newblock 2025.
\newblock URL \url{https://arxiv.org/abs/2502.02737}.

\bibitem[Loshchilov and Hutter(2019)]{loshchilov2019decoupledweightdecayregularization}
Ilya Loshchilov and Frank Hutter.
\newblock Decoupled weight decay regularization.
\newblock 2019.
\newblock URL \url{https://arxiv.org/abs/1711.05101}.

\bibitem[Morris et~al.(2023)Morris, Zhao, Chiu, Shmatikov, and Rush]{morris2023languagemodelinversion}
John~X. Morris, Wenting Zhao, Justin~T. Chiu, Vitaly Shmatikov, and Alexander~M. Rush.
\newblock Language model inversion.
\newblock 2023.
\newblock URL \url{https://arxiv.org/abs/2311.13647}.

\bibitem[Micikevicius et~al.(2018)Micikevicius, Narang, Alben, Diamos, Elsen, Garcia, Ginsburg, Houston, Kuchaiev, Venkatesh, and Wu]{micikevicius2018mixedprecisiontraining}
Paulius Micikevicius, Sharan Narang, Jonah Alben, Gregory Diamos, Erich Elsen, David Garcia, Boris Ginsburg, Michael Houston, Oleksii Kuchaiev, Ganesh Venkatesh, and Hao Wu.
\newblock Mixed precision training.
\newblock 2018.
\newblock URL \url{https://arxiv.org/abs/1710.03740}.

\bibitem[Dettmers et~al.(2022)Dettmers, Lewis, Belkada, and Zettlemoyer]{dettmers2022llmint88bitmatrixmultiplication}
Tim Dettmers, Mike Lewis, Younes Belkada, and Luke Zettlemoyer.
\newblock Llm.int8(): 8-bit matrix multiplication for transformers at scale.
\newblock 2022.
\newblock URL \url{https://arxiv.org/abs/2208.07339}.

\bibitem[Micikevicius et~al.(2022)Micikevicius, Stosic, Burgess, Cornea, Dubey, Grisenthwaite, Ha, Heinecke, Judd, Kamalu, Mellempudi, Oberman, Shoeybi, Siu, and Wu]{micikevicius2022fp8formatsdeeplearning}
Paulius Micikevicius, Dusan Stosic, Neil Burgess, Marius Cornea, Pradeep Dubey, Richard Grisenthwaite, Sangwon Ha, Alexander Heinecke, Patrick Judd, John Kamalu, Naveen Mellempudi, Stuart Oberman, Mohammad Shoeybi, Michael Siu, and Hao Wu.
\newblock Fp8 formats for deep learning.
\newblock 2022.
\newblock URL \url{https://arxiv.org/abs/2209.05433}.

\bibitem[Rumelhart(1986)]{rumelhart1986parallel}
David~E Rumelhart.
\newblock Parallel distributed processing: Explorations in the microstructure of cognition.
\newblock \emph{The MIT Press google schola}, 2:\penalty0 359--366, 1986.

\bibitem[Sejnowski and Rosenberg(1987)]{sejnowski1987parallel}
Terrence~J Sejnowski and Charles~R Rosenberg.
\newblock Parallel networks that learn to pronounce english text.
\newblock \emph{Complex systems}, 1\penalty0 (1):\penalty0 145--168, 1987.

\bibitem[Huh et~al.(2024)Huh, Cheung, Wang, and Isola]{huh2024platonicrepresentationhypothesis}
Minyoung Huh, Brian Cheung, Tongzhou Wang, and Phillip Isola.
\newblock The platonic representation hypothesis.
\newblock 2024.
\newblock URL \url{https://arxiv.org/abs/2405.07987}.

\bibitem[Bahdanau et~al.(2016)Bahdanau, Cho, and Bengio]{bahdanau2016neuralmachinetranslationjointly}
Dzmitry Bahdanau, Kyunghyun Cho, and Yoshua Bengio.
\newblock Neural machine translation by jointly learning to align and translate.
\newblock 2016.
\newblock URL \url{https://arxiv.org/abs/1409.0473}.

\bibitem[Kingma and Ba(2017)]{kingma2017}
Diederik~P. Kingma and Jimmy Ba.
\newblock Adam: A method for stochastic optimization.
\newblock 2017.
\newblock URL \url{https://arxiv.org/abs/1412.6980}.

\bibitem[Paszke et~al.(2019)Paszke, Gross, Massa, Lerer, Bradbury, Chanan, Killeen, Lin, Gimelshein, Antiga, Desmaison, Köpf, Yang, DeVito, Raison, Tejani, Chilamkurthy, Steiner, Fang, Bai, and Chintala]{paszke2019pytorchimperativestylehighperformance}
Adam Paszke, Sam Gross, Francisco Massa, Adam Lerer, James Bradbury, Gregory Chanan, Trevor Killeen, Zeming Lin, Natalia Gimelshein, Luca Antiga, Alban Desmaison, Andreas Köpf, Edward Yang, Zach DeVito, Martin Raison, Alykhan Tejani, Sasank Chilamkurthy, Benoit Steiner, Lu~Fang, Junjie Bai, and Soumith Chintala.
\newblock Pytorch: An imperative style, high-performance deep learning library.
\newblock 2019.
\newblock URL \url{https://arxiv.org/abs/1912.01703}.

\bibitem[Wolf et~al.(2020)Wolf, Debut, Sanh, Chaumond, Delangue, Moi, Cistac, Rault, Louf, Funtowicz, Davison, Shleifer, von Platen, Ma, Jernite, Plu, Xu, Scao, Gugger, Drame, Lhoest, and Rush]{wolf-etal-2020-transformers}
Thomas Wolf, Lysandre Debut, Victor Sanh, Julien Chaumond, Clement Delangue, Anthony Moi, Pierric Cistac, Tim Rault, Rémi Louf, Morgan Funtowicz, Joe Davison, Sam Shleifer, Patrick von Platen, Clara Ma, Yacine Jernite, Julien Plu, Canwen Xu, Teven~Le Scao, Sylvain Gugger, Mariama Drame, Quentin Lhoest, and Alexander~M. Rush.
\newblock Transformers: State-of-the-art natural language processing.
\newblock In \emph{Proceedings of the 2020 Conference on Empirical Methods in Natural Language Processing: System Demonstrations}, pages 38--45, Online, October 2020. Association for Computational Linguistics.
\newblock URL \url{https://www.aclweb.org/anthology/2020.emnlp-demos.6}.

\bibitem[Gugger et~al.(2022)Gugger, Debut, Wolf, Schmid, Mueller, Mangrulkar, Sun, and Bossan]{accelerate}
Sylvain Gugger, Lysandre Debut, Thomas Wolf, Philipp Schmid, Zachary Mueller, Sourab Mangrulkar, Marc Sun, and Benjamin Bossan.
\newblock Accelerate: Training and inference at scale made simple, efficient and adaptable.
\newblock \url{https://github.com/huggingface/accelerate}, 2022.

\bibitem[Zhao et~al.(2023)Zhao, Gu, Varma, Luo, Huang, Xu, Wright, Shojanazeri, Ott, Shleifer, Desmaison, Balioglu, Damania, Nguyen, Chauhan, Hao, Mathews, and Li]{zhao2023pytorchfsdpexperiencesscaling}
Yanli Zhao, Andrew Gu, Rohan Varma, Liang Luo, Chien-Chin Huang, Min Xu, Less Wright, Hamid Shojanazeri, Myle Ott, Sam Shleifer, Alban Desmaison, Can Balioglu, Pritam Damania, Bernard Nguyen, Geeta Chauhan, Yuchen Hao, Ajit Mathews, and Shen Li.
\newblock Pytorch fsdp: Experiences on scaling fully sharded data parallel.
\newblock 2023.
\newblock URL \url{https://arxiv.org/abs/2304.11277}.

\bibitem[Lhoest et~al.(2021)Lhoest, Villanova~del Moral, Jernite, Thakur, von Platen, Patil, Chaumond, Drame, Plu, Tunstall, Davison, {\v{S}}a{\v{s}}ko, Chhablani, Malik, Brandeis, Le~Scao, Sanh, Xu, Patry, McMillan-Major, Schmid, Gugger, Delangue, Matussi{\`e}re, Debut, Bekman, Cistac, Goehringer, Mustar, Lagunas, Rush, and Wolf]{lhoest-etal-2021-datasets}
Quentin Lhoest, Albert Villanova~del Moral, Yacine Jernite, Abhishek Thakur, Patrick von Platen, Suraj Patil, Julien Chaumond, Mariama Drame, Julien Plu, Lewis Tunstall, Joe Davison, Mario {\v{S}}a{\v{s}}ko, Gunjan Chhablani, Bhavitvya Malik, Simon Brandeis, Teven Le~Scao, Victor Sanh, Canwen Xu, Nicolas Patry, Angelina McMillan-Major, Philipp Schmid, Sylvain Gugger, Cl{\'e}ment Delangue, Th{\'e}o Matussi{\`e}re, Lysandre Debut, Stas Bekman, Pierric Cistac, Thibault Goehringer, Victor Mustar, Fran{\c{c}}ois Lagunas, Alexander Rush, and Thomas Wolf.
\newblock Datasets: A community library for natural language processing.
\newblock In \emph{Proceedings of the 2021 Conference on Empirical Methods in Natural Language Processing: System Demonstrations}, pages 175--184, Online and Punta Cana, Dominican Republic, November 2021. Association for Computational Linguistics.
\newblock URL \url{https://aclanthology.org/2021.emnlp-demo.21}.

\bibitem[Su et~al.(2023)Su, Lu, Pan, Murtadha, Wen, and Liu]{su2023roformerenhancedtransformerrotary}
Jianlin Su, Yu~Lu, Shengfeng Pan, Ahmed Murtadha, Bo~Wen, and Yunfeng Liu.
\newblock Roformer: Enhanced transformer with rotary position embedding.
\newblock 2023.
\newblock URL \url{https://arxiv.org/abs/2104.09864}.

\bibitem[Ancona et~al.(2018)Ancona, Ceolini, Öztireli, and Gross]{ancona2018betterunderstandinggradientbasedattribution}
Marco Ancona, Enea Ceolini, Cengiz Öztireli, and Markus Gross.
\newblock Towards better understanding of gradient-based attribution methods for deep neural networks.
\newblock 2018.
\newblock URL \url{https://arxiv.org/abs/1711.06104}.

\end{thebibliography}
\beginsupplement{}
\section{Appendix}

    \subsection{Experimental Details}

    Training details are found in \ref{autosection}, and we provide more here. We construct models using Pytorch \citep{paszke2019pytorchimperativestylehighperformance}, train models using (modified) Hugging Face Transformers \citep{wolf-etal-2020-transformers} trainers with accelerate \citep{accelerate} integration. Most models were trained using Distributed Data Parallel with Pytorch-native automated fp16/fp32 mixed-precision (AMP) which stores weights and parameters in 32-bit precision and downcasts during matrix multiplication, although some were trained using bf16/fp32 mixed precision. A very small number of models were trained using FSDP \citep{zhao2023pytorchfsdpexperiencesscaling}, but training throughputs using this approach were not compared to DDP. We use Datasets \citep{lhoest-etal-2021-datasets} to store, load, and augment data, and typically train using pre-tokenized and pre-processed (padded, entropy estimated etc.) datasets for speed. We train using AdamW with a 500-step warmp-up, a maximum learning rate of $\eta=2 * 10^{-4}$ for transformers and $\eta=5*10^{-4}$ for mixers, with linear learning rate decay to 0 over the full training run (200k steps, or rarely 500k steps). For $n_{ctx}=512$ experiments, the batch size is $b=128$ and for $n_{ctx}=1024$ we use $b=64$. 

   \begin{figure}[h]
        \centering
        \includegraphics[width=0.95\textwidth]{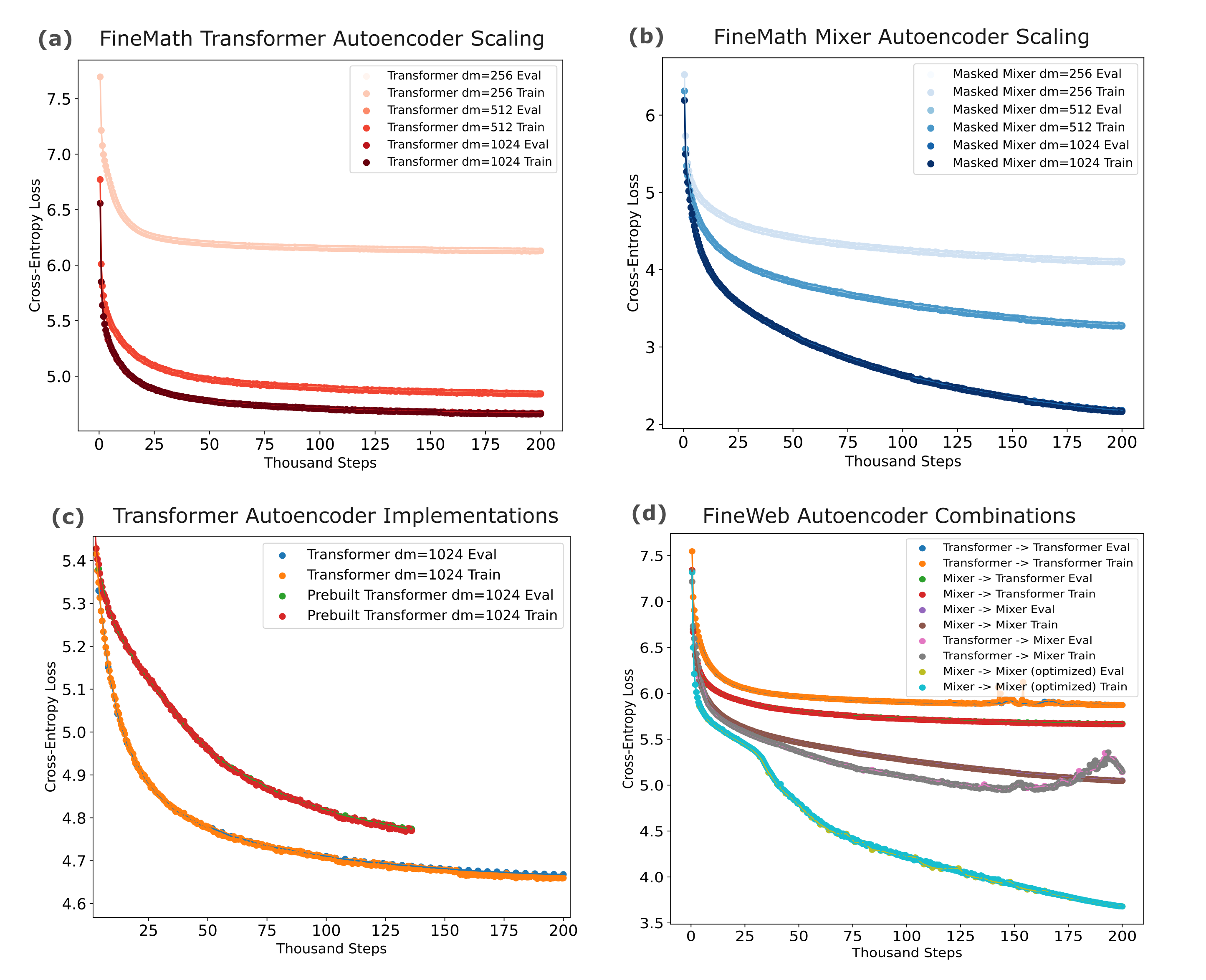}
        \caption{Repeated embedding autoencoder training characterization. (a) Transformers and (b) Masked Mixer-based autoencoder training curves trained on FineMath 4+, note the difference in loss scaling. (c) Llama transformer module-based autoencoders (with independent embeddings) slightly outperform autoencoders implemented with prebuilt causal Llama models, ruling out implementation errors. (d) Autoencoders with transformer decoders but not encoders are inefficient to train.}
        \label{figs1}
    \end{figure}

    \subsection{Why Transformers Struggle with Repeated Embeddings} \label{sections1}

    We reasoned that the nature of the transformer decoder's self-attention transformations, we consider the formal definition of this transformation as shown in \ref{eq22}, where the $Q, K, V$ are matrices of packed vectors projected from each token embedding.

    \begin{equation}
    A(Q, K, V) = \mathrm{softmax} \left( \frac{QK^T}{\sqrt(d_k)} \right)V
    \label{eq22}
    \end{equation}
    
    By definition, introduction of repeat embeddings means that inputs at each token position are identical. Considering the self-attention operations with respect to token indices, we can ignore the $d_k$ scaling factor and express Equation \ref{eq22} as Equation \ref{eq23}, which gives the attention value between the query projection for token $i$ and key and value projections at index $j$, for ${i, j \in n}$. 
    
    \begin{equation}
    A(q_i, k_j, v_j) = \frac{\exp \left( (q_i \cdot k_j) v_j \right)}{ \sum_n \exp \left( (q_i \cdot k_n) v_n \right)}
    \label{eq23}
    \end{equation}
    
    With repeated embedding introduction, and because as the projection weight matrices $W_k, W_q, W_v$ are identical for all tokens, the we have the following:
    $k_i = k_j, \; q_i = q_j, \; v_i = v_j \forall i, j$ and therefore $q_i \cdot k_j = q_i \cdot k_l \; \; \forall i, j, l$ such that $A(q_i, k_j, v_j) = A(q_i, k_l, v_l) \; \forall i, j, l$ and necessarily that output activations from the attention layer are identical across all token embeddings.  As norms and feedforward layers in the transformer are identical between token indices, outputs will be identical as well.
    
    This provides a clear reason as to why a transformer decoder would be incapable of generating language from a repeated embedding, but omits an integral part of our transformer architecture: Llama-style transformers apply positional encoding (Rotary Positional Encoding \citep{su2023roformerenhancedtransformerrotary} for our models) before each self-attention operation such that the embeddings at each position are actually unique, assuming that the positional encoding is itself unique for the token indices used. Thus is is not strictly correct to point to identical activations due to self-attention as being the cause of the poor transformer training for repeat-embedding autoencoders, but we conjectured that positional encoding might not impart sufficient uniqueness to each embedding for the transformer decoder to adequately make use of the embedding's informational content. 

    \begin{figure}[h]
        \centering
        \includegraphics[width=0.95\textwidth]{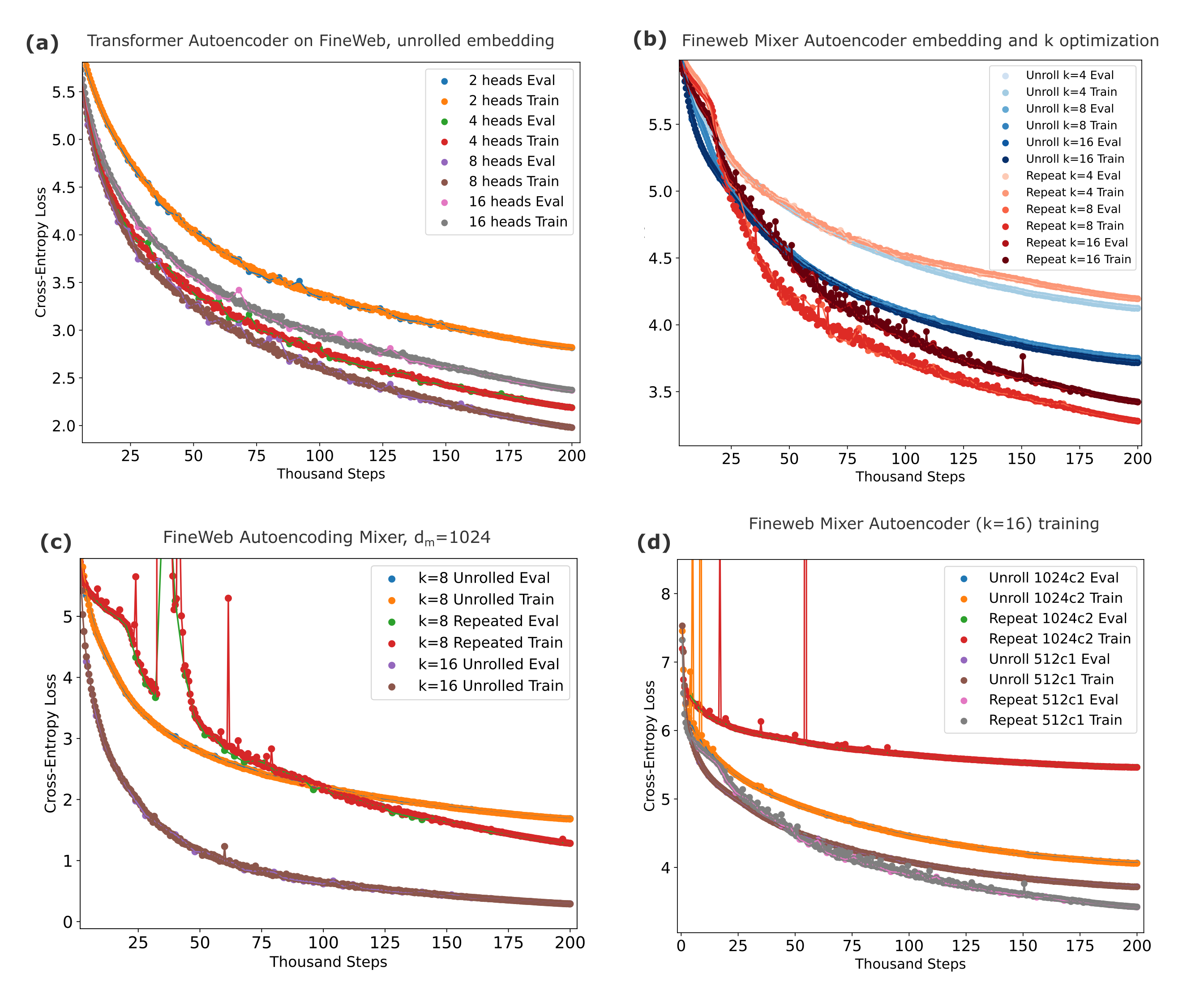}
        \caption{Autoencoder training efficiency characterization with architectural variations. All models are $d_m=512, n_l=16, n_{ctx}=512$ unless otherwise noted. For $d_m=1024$ models, unrolling is performed middle-out rather than first-forward, where Equation \ref{eq6} is substituted for $E(x) = W(x_{m :m + s} \circ x_{0: \;\max (0, \; m + s - d_m/2)}) + \beta$. 200,000 training steps corresponds to approximately 13 billion tokens.}
        \label{figs2}
    \end{figure}

    \begin{figure}[h]
        \centering
        \includegraphics[width=0.95\textwidth]{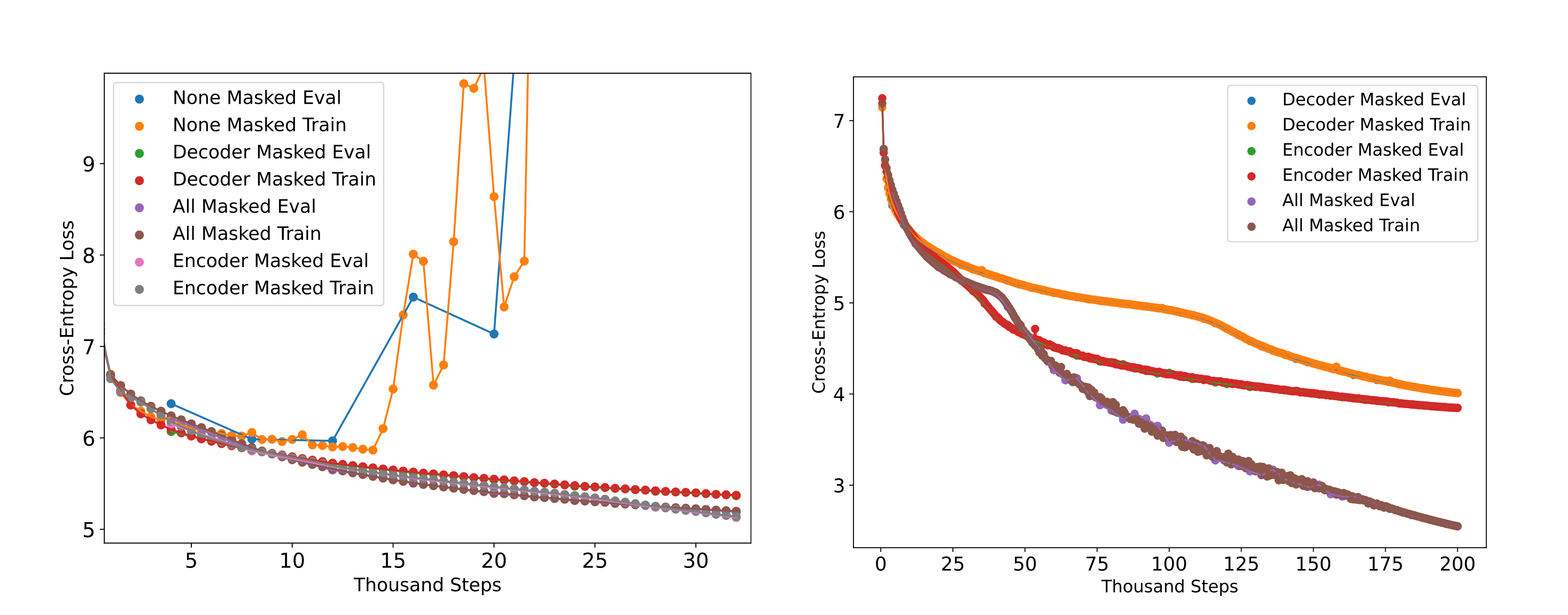}
        \caption{Causal masking is required for stable mixer-based autoencoder training, and masking (training performed on Fineweb).}
        \label{figs3}
    \end{figure}

    \begin{figure}[h]
        \centering
        \includegraphics[width=0.95\textwidth]{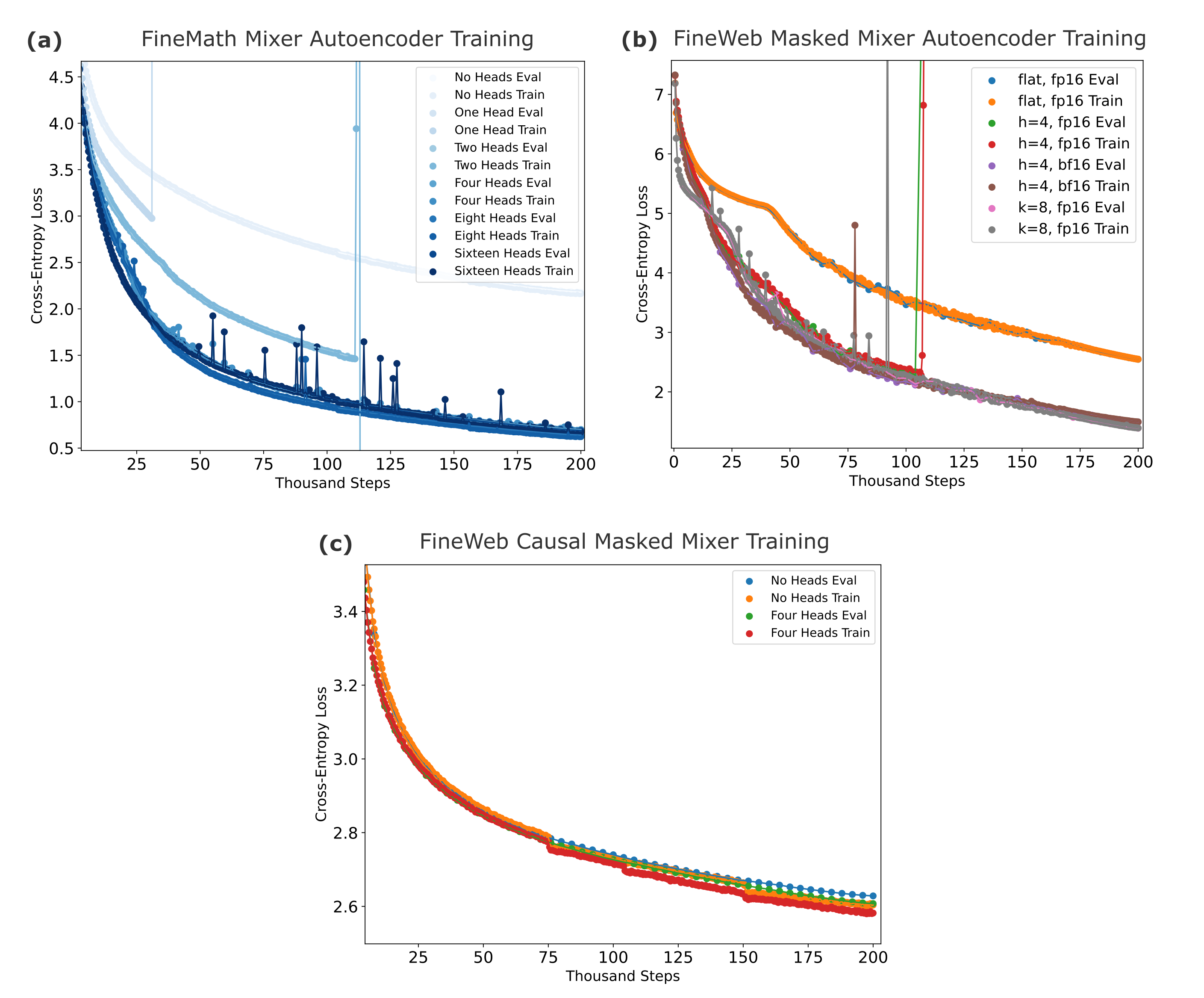}
        \caption{Increasing the head number or kernel size results in substantial improvements in autoencoder but not causal model training efficiency. Models are $d_m=1024, n_l=16$. For multi-headed mixers, the number of parameters per layer is $p = n_h * n_{ctx}^2 + 2d_m^2$ for $n_h$ heads per layer with each head dimension being $d_m / n_h$. For multi-kernel mixers, each layer contains $p=k * n_{ctx}^2$ parameters such that for the default case $k=1$ without heads we have $p=n_{ctx}^2$}
        \label{figs4}
    \end{figure}

    \begin{figure}[h]
        \centering
        \includegraphics[width=0.95\textwidth]{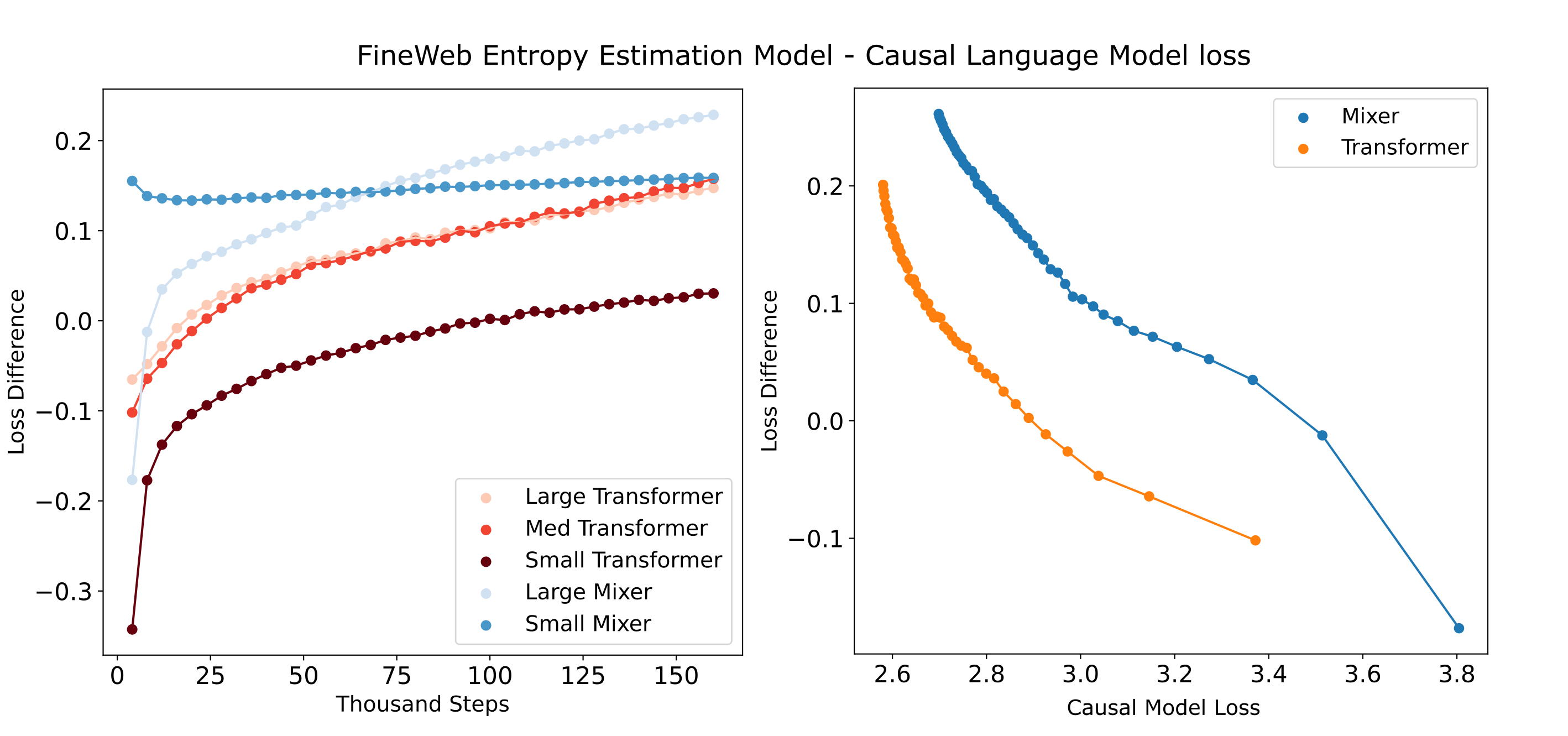}
        \caption{Entropy estimation model size scaling (all with $d_e=64$). Left, loss difference (clm - eem) per training step (over 13.1 billion tokens). Right, loss difference per causal language model loss value (note the exponential growth) for medium transformer and large mixer. For transformers: small models are $d_e=128$ encoder dimension and $d_d = 256$ decoder dimension with $n_l=8$ layers each, the medium transformer is $d_e=256, d_d=512, n_l=16$ and large $d_e = 512, d_m=1024, n_l=24$. Small mixer is $d_e=128, d_d=512, n_l=8$ and large $d_e=256, d_d=1024, n_l=16$.}
        \label{figs5}
    \end{figure}

    \begin{figure}[h]
        \centering
        \includegraphics[width=0.8\textwidth]{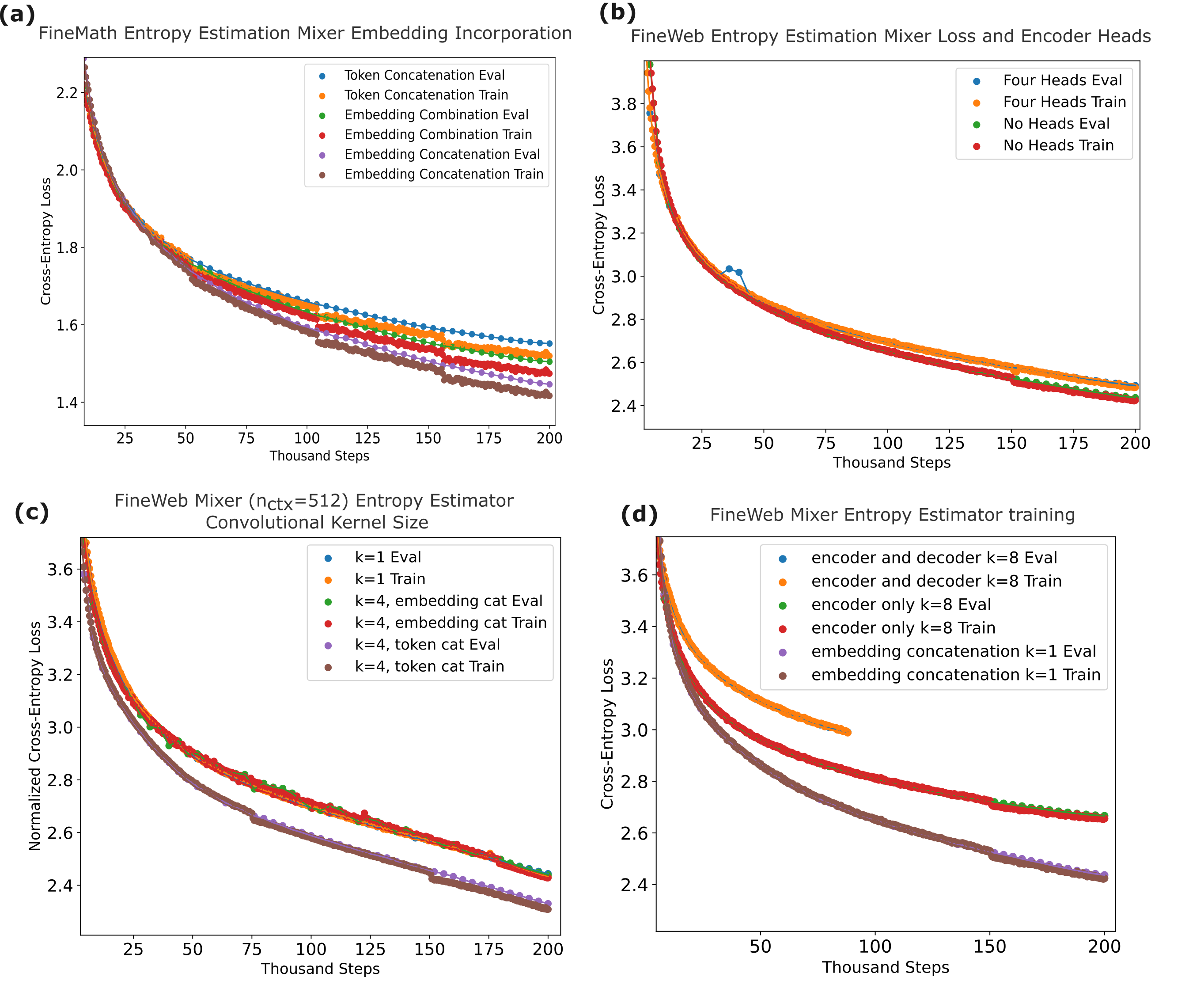}
        \caption{Head and kernel optimizations that helped large autoencoders don't increase entropy estimation model training efficiency. 200k steps are approximately 13 billion tokens.}
        \label{figs6}
    \end{figure}

    \begin{figure}[h]
        \centering
        \includegraphics[width=0.9\textwidth]{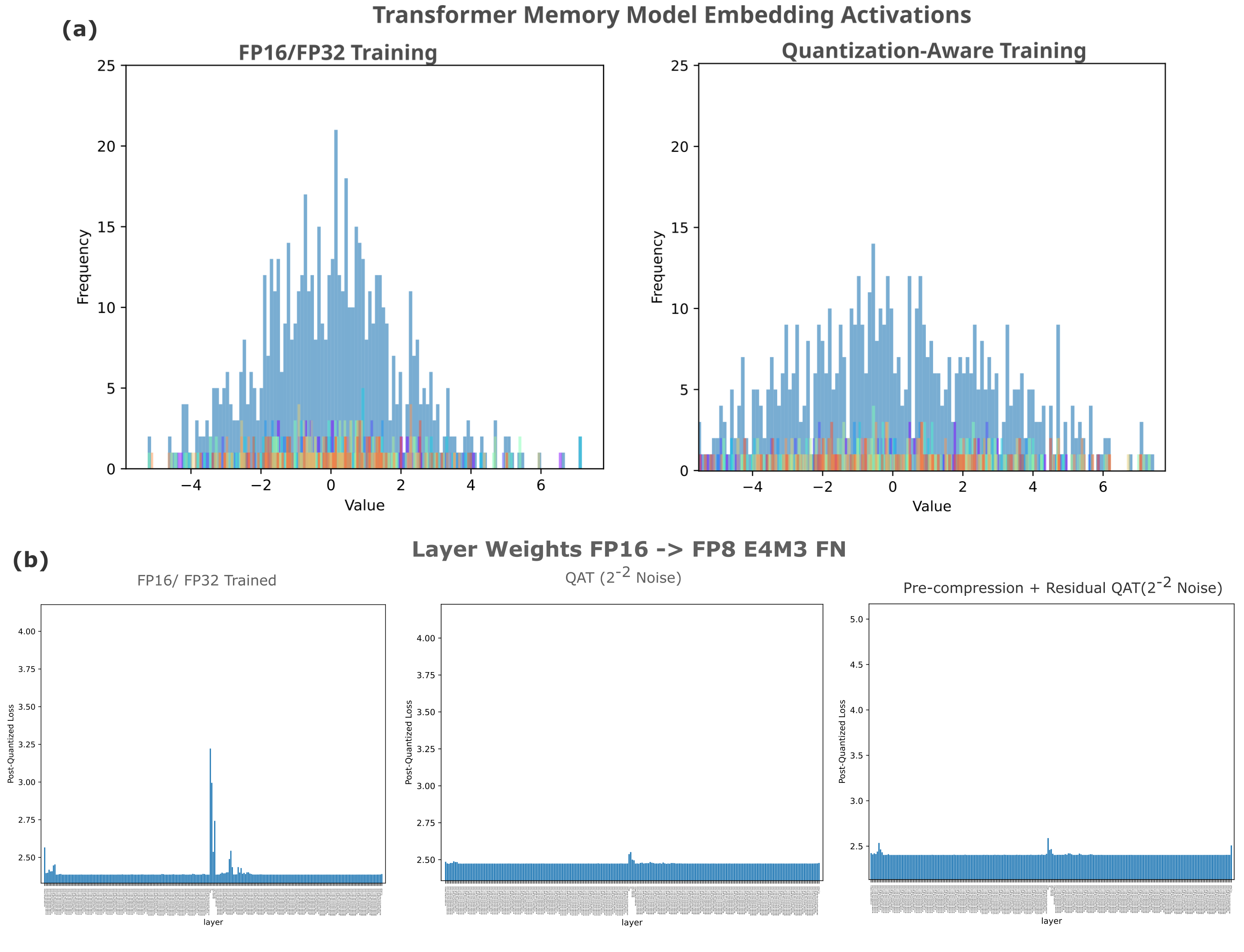}
        \caption{Quantization-aware training characterizations. (a) Activation distributions (colors indicate one embedding, blue the aggregate of many embeddings) with and without $2^{-2}$ noise addition. (b) Weights are also trained to become quantization-insensitive upon injection of activation noise.}
        \label{figs7}
    \end{figure}

   \subsection{Proxy methods for entropy estimation} \label{proxysection}

   We tested whether attribution methods would allow for an approximation of entropy estimations, and in particular whether one-pass attributions would allow one to forego the $n_{ctx}$ forward passes for each sample. We investigated attribution via masking of the embedding and measuring the change in output upon doing so, inspired by occlusion-based attribution where input elements are masked to judge their contributions to the output \citep{ancona2018betterunderstandinggradientbasedattribution}. We calculate occlusion as shown in Equation \ref{eq20}, where $W_{wte}$ is the decoder's word token embedding transformation, not the encoder's, and $\circ$ signifies concatenation (in this case in the sequence dimension), $\mathbf{0}$ the zero vector, $\theta_e$ the encoder model, and $\theta_d$ the decoder. In addition to occluding the memory input, we apply an attention mask to that input as well for transformer models.
    
    \begin{equation}
    \begin{split}
    &x = O(x, \theta_e) \circ W_{wte}x \\
    &x_o = \mathbf{0} \circ W_{wte}x \\
    &Attr(x) = m(O(x, \theta_d), O(x_o, \theta_d))
    \label{eq20}
    \end{split}
    \end{equation}
    
    We investigate two metrics $m$ used to determine attributions: the $L^1$ norm given in Equation (\ref{eq21}) and the cosine similarity distance given in Equation (\ref{eq40}). For Equation (\ref{eq21}) where $i$ is indexed in the embedding dimension. Here we actually use the logit activations rather than the embeddings, so effectively $m_{l^1}$ measures the Manhattan metric between the decoder's logits with versus without the encoder's embedding. In both metrics for transformers we can also remove the embedding information using an attention mask. For Cosine similarity, we compute the complement of this value as shown in Equation \ref{eq40} where $j$ iterates on the sequence dimension, and $\mathrm{max}, \mathrm{min}$ are computed on this dimension as well. We mask all pad input elements during this normalization process, such that these are assigned infinite values for minimum computation and zero values for maximum computation (the norms of $y$ values from trained models are usually >10000, and none were observed to have zero distance in part due to their high dimensionality).
    
    \begin{equation}
    m_{l^1}(O(x, \theta_d), O(x_o, \theta_d)) = || O(x, \theta_d) - O(x_o, \theta_d) ||_1 = \sum_i | O_i(x, \theta_d) - O_i(x_o, \theta_d) | 
    \label{eq21}
    \end{equation}

    \begin{equation}
    m_{cosine}(O(x, \theta_d), O(x_o, \theta_d)) =  1 -  \frac{O(x, \theta_d) \cdot O(x_o, \theta_d)}{|| O(x, \theta_d) || \; ||  O(x_o, \theta_d) ||}
    \label{eq40}
    \end{equation}
    
    $L^1$ norms are sensitive to changes in scale between samples, which can be a problem as gradient descent is normally calculated batchwise such that scale inequalities between samples in a batch lead to biases in gradient magnitude once weights are applied. To normalize all token attributions to take values in $[0, 1]$, we use a simple per-sample linear minmax approach as shown in (\ref{eq41}). Cosine similarity has the advantage of not needing to be normalized, as the range is $m_{cosine}(O(x, \theta_d), O(x_o, \theta_d)) \in [0, 2]$ with nearly all values in $[0, 1]$ for sufficiently high-dimensional output vectors. Correlations obtained between these attribution methods and other entropy estimation methods are given in Table \ref{tables1}.

    \begin{equation}
    N_{minmax}(y) = \frac{y_j - \mathrm{min} \; y}{\mathrm{max}\; y - \mathrm{min} \; y }
    \label{eq41}
    \end{equation}

    \begin{figure}[h]
        \centering
        \includegraphics[width=0.99\textwidth]{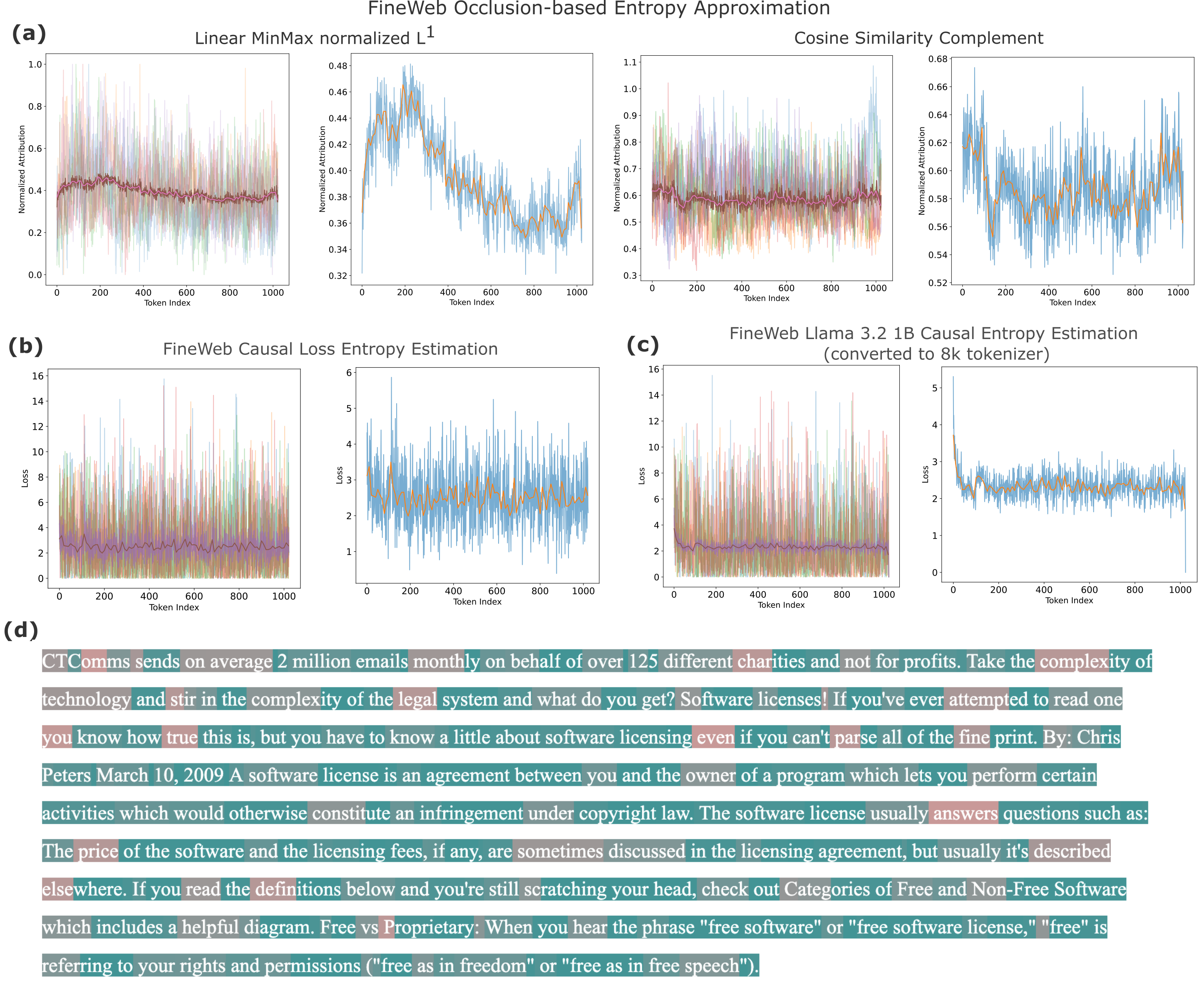}
        \caption{Per-token entropy estimation displays some statistical features that are aligned with prior expectations: higher entropy at the start of documents, and higher entropy at the start of words for words tokenized by more than one token. (a) Average occlusion-based entropy per token index. (b) Average Causal language model entropy per token index. (c) Llama 3.1 (1b) -derived entropy per token index, after conversion to the 8k tokenizer. (d) Causal language model-derived entropies per token of a FineWeb document excerpt, with red denoting higher and green denoting lower entropy.}
        \label{figs8}
    \end{figure}

    \begin{center}
    \begin{table}[H]
    \begin{center}
    \renewcommand{\arraystretch}{1.2}
    \begin{tabular}{||l c c c c ||} 
     \hline
      y vs x & m & x & b &\\ [0.5ex] 
     \hline\hline
      $L^1$ occlusion, cosine occlusion & 0.9566 & 0.1431 &  0.4195 &\\
     \hline
      $L^1$ occlusion, CLM loss & 0.0063 & 0.3894 &0.0107 & \\ 
     \hline
     Large embedding $L^1$ occlusion, CLM loss & 0.0172 & 0.4152 & 0.0688 & \\
     \hline 
     Large embedding $L^1$ occlusion, $L^1$ loss & 0.2308 & 0.3691 & 0.0424 & \\
     \hline
     Very large embedding $L^1$ occlusion, $L^1$ loss & 0.2680 & 0.3629 & 0.0574 & \\
     \hline
    \end{tabular}
    \end{center}
    \vspace{0.1cm}
    \caption{Correlations between entropy estimations and proxy estimation methods per token.}
    \label{tables1}
    \end{table}
    \end{center}

\end{document}